\documentclass[11pt]{article}
\usepackage[round,authoryear]{natbib}
\usepackage{metacircle}
\usepackage{amsmath,amsfonts,bm}

\def\eqref#1{equation~\ref{#1}}

\def\1{\bm{1}}

\DeclareMathAlphabet{\mathsfit}{\encodingdefault}{\sfdefault}{m}{sl}
\SetMathAlphabet{\mathsfit}{bold}{\encodingdefault}{\sfdefault}{bx}{n}

\usepackage{amsthm}
\usepackage{capt-of}
\usepackage{placeins}
\usepackage{indentfirst}
\newtheorem{theorem}{Theorem}[section]

\newtheorem{proposition}[theorem]{Proposition}

\theoremstyle{definition}

\theoremstyle{remark}

\newcommand{\loss}{\mathcal{L}}
\newcommand{\hessian}{\mathbf{H}}

\definecolor{questionred}{RGB}{139,0,24}
\definecolor{questionback}{RGB}{244,238,239}
\newtcolorbox{researchquestion}[1]{
  enhanced,
  colback=questionback,
  colframe=questionred,
  colbacktitle=questionred,
  coltitle=white,
  fonttitle=\bfseries,
  title={#1},
  boxrule=1.1pt,
  arc=2mm,
  left=3mm,
  right=3mm,
  top=2mm,
  bottom=2mm,
  before skip=0.7\baselineskip,
  after skip=0.9\baselineskip
}

\title{Rethinking Activation Functions:\\ Reducing Bias Can Improve Learning}
\author{Zirui Peng$^{1,*}$, Heyi Zheng$^{2}$, Ziming Liu$^{3,4}$\\
  {\small $^{1}$Department of Physics, Peking University}\\
  {\small $^{2}$Department of Mathematics, Peking University}\\
  {\small $^{3}$College of AI, Tsinghua University}\quad
  {\small $^{4}$MetaCircle}\\
  {\small $^{*}$Email: \href{mailto:fredpengphy@stu.pku.edu.cn}{fredpengphy@stu.pku.edu.cn}}}
\date{}
\papertype{Preprint}
\paperdate{October 2026}

\begin{document}

\maketitle

\begin{metaabstract}
Many widely used activation functions, including ReLU, GELU, Swish, and Sigmoid, produce predominantly positive outputs, resulting in a nonzero post-activation mean that we term \emph{activation bias}. Despite its ubiquity, the effect of activation bias on learning dynamics remains poorly understood. Using controlled models, we show that activation bias accelerates initial loss reduction but is accompanied by overshooting, non-smooth trajectories, and a steeper loss landscape. This early advantage also coincides with the model's use of task-relevant information encoded in the activation bias component. In contrast, reducing activation bias slows initial learning but ultimately yields a lower final loss. Experiments on larger vision and language models exhibit a similar pattern: smaller activation bias is associated with better late-stage performance. Thus, this seemingly minor property can substantially redirect learning and warrants explicit consideration in future activation function design.
\end{metaabstract}

\begin{figure}[h]
  \centering
  \includegraphics[width=\linewidth]{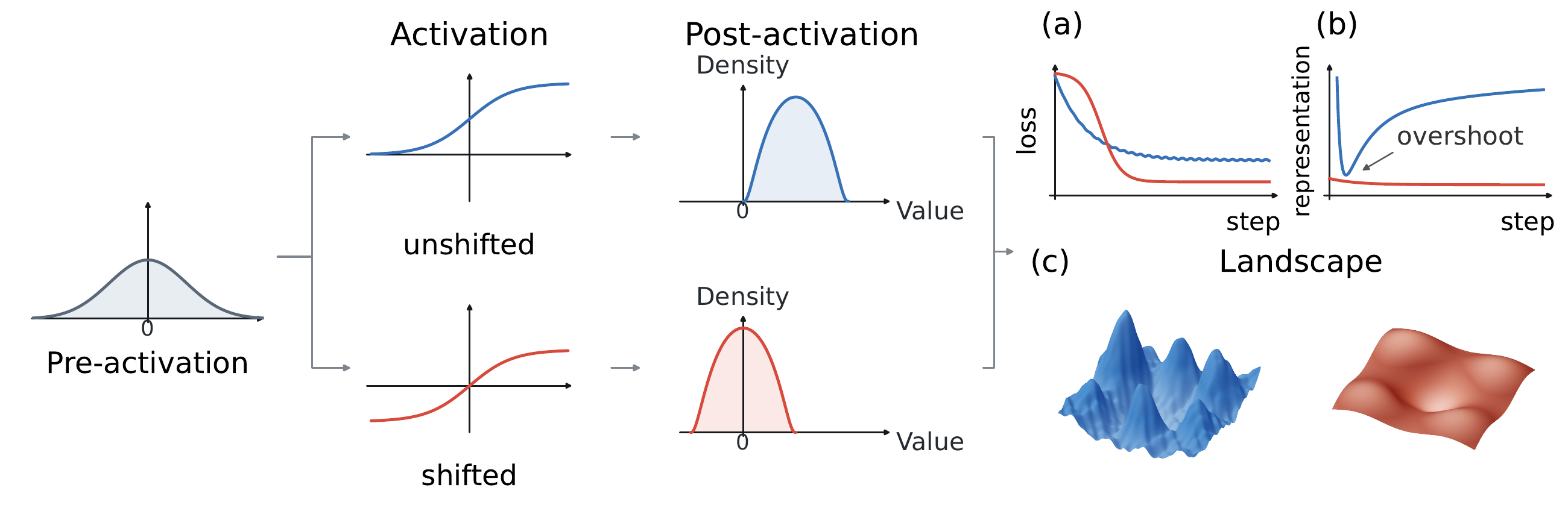}
  \caption{Overview of activation bias and its effects on learning. Blue denotes the
unshifted activation, and red denotes the shifted activation. Reducing
activation bias changes the post-activation distribution, model performance,
hidden-state trajectory, and loss landscape.}
  \label{fig:activation-bias-overview}
\end{figure}

\section{Introduction}
\label{sec:introduction}

Many common activation functions, including ReLU, GELU, Swish, and Sigmoid,
favor positive outputs while suppressing or excluding negative ones. These
asymmetries typically yield a nonzero post-activation mean. For
a pre-activation $z$, we define the \emph{activation bias} of an activation
function $\phi$ as the mean of its post-activation:
\begin{equation}
b_\phi \equiv \mathbb{E}[\phi(z)].
\label{eq:activation-bias}
\end{equation}
These functions are widely used and perform well across diverse tasks. Their
success, however, does not establish how the post-activation bias they induce
affects learning dynamics.

Prior work connects activation design to trainability through signal
propagation, dynamical isometry, critical initialization, and normalization
\citep{clevert2016fast,brock2021characterizing,qi2020deep,
lesser2026competing,bjorck2018understanding,santurkar2018how}. These studies
primarily address propagation stability or optimization conditioning, leaving
open how activation bias affects learning over time and whether the bias built
into common activations is optimal for learning. We therefore ask:

\begin{researchquestion}{Q: How does activation bias shape learning dynamics?}
(i) How does it alter training dynamics? (ii) Does the activation bias component
carry useful information, and how does the model use that information during
learning?
\end{researchquestion}

We first investigate the effect of activation bias on training dynamics in a
controlled deep residual teacher--student model. For each activation function,
we use a fixed post-activation shift to control the degree of activation bias. Both
students learn from the same fixed teacher, with identical initial parameters,
training data, and optimization settings; they differ only in whether the shift
is applied.

Reducing activation bias slows initial loss reduction, but the shifted student
eventually overtakes its unshifted counterpart and attains a lower final loss.
Theory and experiments associate this reversal with smoother
optimization dynamics, distinct parameter-update directions, and a flatter
loss landscape, reflected in substantially smaller Hessian eigenvalues.

We then use a controlled image-reconstruction task to test whether the activation bias
component carries information the model can exploit. Early in training, decoding
this component alone recovers recognizable image content. Its output later
departs from the full-representation reconstruction, suggesting that the model
relies less on this information as training proceeds. This transient use
reveals a non-smooth feature-learning trajectory and is largely suppressed when
activation bias is reduced.

Finally, we test whether these findings persist in larger models and across a
wider range of tasks.
These models exhibit analogous non-smooth
dynamics under the original activations, while reducing activation bias yields
comparable or better late-stage performance. Thus, this seemingly small
activation property can produce pronounced changes in learning across
architectures and tasks.

\begin{researchquestion}{Our contributions are the following:}
\begin{itemize}[leftmargin=5mm,labelsep=1.5mm,itemsep=2pt,topsep=1pt,
  parsep=0pt,partopsep=0pt]
  \item We use a controlled model to investigate activation bias and reveal its
  effect: it accelerates early learning but impedes later progress.
  \item We characterize how activation bias shapes optimization dynamics and
  how the model uses information carried by the activation bias component during learning.
  \item We observe similar learning dynamics and improved late-stage
  performance with reduced activation bias across larger
  models and a wider range of tasks.
\end{itemize}
\end{researchquestion}

To our knowledge, this is the first study to examine how activation bias shapes
learning dynamics throughout training. Our results show that a small change in
activation bias can redirect optimization, yield smoother learning
trajectories, and improve late-stage performance. This sensitivity calls for treating
activation bias as an explicit property in activation-function design rather
than an incidental by-product.

\section{Analysis of Training Dynamics}
\label{sec:toy-model}
\label{sec:landscape-analysis}

\subsection{Controlled Teacher--Student Setup}
\label{sec:toy-setup}

\begin{figure}[t]
  \centering
  \includegraphics[width=\linewidth]{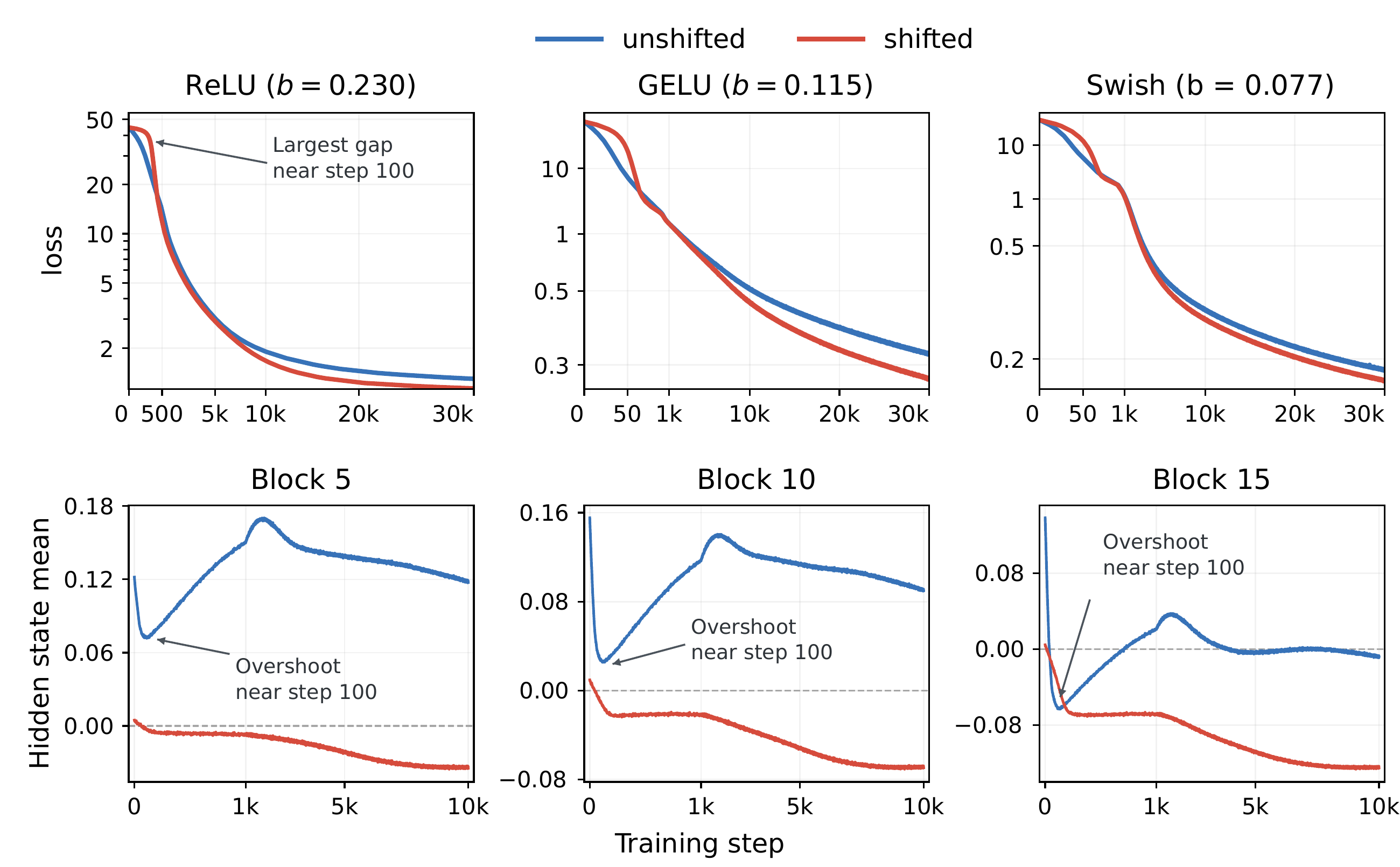}
  \caption{Activation bias accelerates early learning, but this advantage
coincides with a pronounced hidden-state overshoot and reverses later in
training. Top: The shifted model learns more slowly initially but ultimately
reaches a lower loss across activation functions. Bottom: The unshifted model's
hidden state overshoots and follows a non-smooth trajectory.}
  \label{fig:toy-loss-overshoot}
\end{figure}

To isolate the effect of activation bias from other factors, we adopt a simple
controlled model. We omit token embeddings, attention, and the output head to
isolate the activation function's role while retaining the pre-normalized MLP
blocks and residual-stream architecture of modern large language models.
For layers $\ell=0,\ldots,L-1$, the model evolves as
\begin{equation}
  \mathbf{x}^{(\ell+1)}
  = \mathbf{x}^{(\ell)}
  + \mathbf{W}_2^{(\ell)}
  \phi\!\left(
    \mathbf{W}_1^{(\ell)}
    \operatorname{LN}\!\left(\mathbf{x}^{(\ell)}\right)
  \right),
  \label{eq:toy-residual-block}
\end{equation}
where $\mathbf{x}^{(\ell)}$ denotes the residual-stream representation,
$\mathbf{W}_1^{(\ell)}$ and $\mathbf{W}_2^{(\ell)}$ are the up- and
down-projection matrices, respectively, $\operatorname{LN}$ denotes pre-normalization, and $\phi$ is the nonlinear
activation function.

We control activation bias by subtracting the same fixed scalar $b$ from the
activation output in every layer:
\begin{equation}
  \phi_{\mathrm{shifted}}(z;b) = \phi(z) - b.
  \label{eq:activation-shift}
\end{equation}

To avoid confounding activation bias with complex data structure that is
difficult to isolate and characterize, we use a teacher--student framework.
The teacher and both students share the architecture in
Equation~\ref{eq:toy-residual-block}.
A fixed teacher using the unshifted activation $\phi$ generates targets
$y=f_{\mathrm{T}}(\mathbf{x})$ for sampled inputs $\mathbf{x}$. We refer to the
student using $\phi$ as the \emph{unshifted student} and the student using
$\phi_{\mathrm{shifted}}$ as the \emph{shifted student}. For each activation,
we estimate $b$ as its expected post-activation output at initialization.
This shift keeps the post-activation mean near
zero over an extended period of training, as detailed in
Appendix~\ref{app:optimal-b}.

\begin{center}
  \begin{minipage}{\linewidth}
  \centering
  \captionof{table}{Activation shifts $b$ used for each activation function.}
  \label{tab:activation-shifts}
  \renewcommand{\arraystretch}{1.15}
  \setlength{\tabcolsep}{8pt}
  \begin{tabular}{c@{\hspace{5pt}}|@{\hspace{5pt}}cccc}
    \hline
    Activation functions & ReLU & GELU & Swish & Sigmoid \\
    \hline
    $b$ & $0.230$ & $0.115$ & $0.077$ & $0.500$ \\
    \hline
  \end{tabular}
  \end{minipage}
\end{center}

\subsection{Early Learning Advantage Coincides with Hidden-State Overshoot}
\label{sec:hidden-state-overshoot}

We first compare the two students through their loss curves. As shown in the
top panel of Figure~\ref{fig:toy-loss-overshoot}, the unshifted student learns
faster initially, but the shifted student eventually overtakes it and converges
to a lower loss. This reversal is notable because the unshifted student uses
exactly the same activation as the teacher yet is ultimately outperformed by
the shifted student.

Motivated by this loss reversal, we first track the residual-stream hidden
state over training. We average its elements for each input and then average
across inputs, yielding a scalar trajectory that tracks the hidden state's
evolution during training. In the bottom panel of
Figure~\ref{fig:toy-loss-overshoot}, the unshifted student's hidden state
overshoots around step $100$, coinciding with its period of faster loss
reduction. The shifted student's hidden state instead evolves smoothly.
These distinct hidden-state trajectories motivate examining whether the two
students also follow different paths in parameter space.

\subsection{Activation Bias Separates Parameter-Update Directions}
\label{sec:toy-landscape-geometry}

To compare parameter trajectories, let $\boldsymbol{\theta}_0$ be the common
initialization and let $\mathbf{g}_{\mathrm{u}}$ and $\mathbf{g}_{\mathrm{s}}$
be the students' initial gradients, where the subscripts $\mathrm{u}$ and
$\mathrm{s}$ denote the unshifted and shifted students, respectively. We
construct the \emph{initial-gradient plane} with orthonormal basis
\begin{equation}
  \mathbf{e}_1
  = \frac{\mathbf{g}_{\mathrm{u}}}{\lVert\mathbf{g}_{\mathrm{u}}\rVert},
  \qquad
  \mathbf{e}_2
  = \frac{(\mathbf{I}-\mathbf{e}_1\mathbf{e}_1^{\top})
  \mathbf{g}_{\mathrm{s}}}
  {\lVert(\mathbf{I}-\mathbf{e}_1\mathbf{e}_1^{\top})
  \mathbf{g}_{\mathrm{s}}\rVert},
  \label{eq:initial-gradient-plane}
\end{equation}
where $\mathbf{e}_1$ follows the unshifted gradient and $\mathbf{e}_2$ is the
orthogonal component of the shifted gradient.

To better understand the early optimization geometry, we visualize each student's loss landscape
on the initial-gradient plane. At each coordinate $(\alpha,\beta)$,
we reconstruct the parameters as
$\boldsymbol{\theta}_0+\alpha\mathbf{e}_1+\beta\mathbf{e}_2$ and evaluate
the corresponding student against the fixed teacher on the evaluation set.
The resulting losses form a heat map of the projected landscape. Projecting
training checkpoints onto the same plane overlays each student's trajectory.
As shown in the left and center panels of
Figure~\ref{fig:toy-landscape}, the two trajectories initially point in nearly
orthogonal directions. The unshifted student initially follows the direction
of strongest descent in its projected landscape and then turns sharply,
revealing a distinctly non-smooth trajectory. The shifted student's
loss varies little on this plane, indicating that its early update develops
largely outside these initial-gradient directions.

To visualize the shifted student's parameter movement outside this plane, we remove the
parts of its parameter displacement along $\mathbf{e}_1$ and $\mathbf{e}_2$:
\begin{equation}
  \mathbf{r}_t
  = (\mathbf{I}-\mathbf{e}_1\mathbf{e}_1^{\top}
  -\mathbf{e}_2\mathbf{e}_2^{\top})
  (\boldsymbol{\theta}(t)-\boldsymbol{\theta}_0).
  \label{eq:orthogonal-update-plane}
\end{equation}
Here, $\mathbf{r}_t$ is the component of the shifted student's parameter
displacement orthogonal to the initial-gradient plane. To obtain two orthogonal directions from the shifted student's parameter displacements outside the initial-gradient plane, we
apply Gram--Schmidt orthogonalization to these displacements and use the first
two resulting basis vectors $\mathbf{n}_1$ and $\mathbf{n}_2$ as its axes.
We then visualize the loss and trajectories on this plane
as above. The
right panel of Figure~\ref{fig:toy-landscape} reveals clear loss variation along
the shifted student's smooth trajectory without a comparable turn.

Together, these views show that activation bias induces qualitatively distinct
parameter-update directions. The two students initially move in nearly
orthogonal directions, and their subsequent updates remain largely separated:
the unshifted trajectory turns within the initial-gradient plane, whereas the
shifted trajectory develops primarily in complementary directions and evolves
smoothly. Further theoretical analysis and experiments on update directions are
provided in Appendix~\ref{app:hessian}.

\begin{figure}[t]
  \centering
  \includegraphics[width=\linewidth,trim=0 0 1.5in 0,clip]{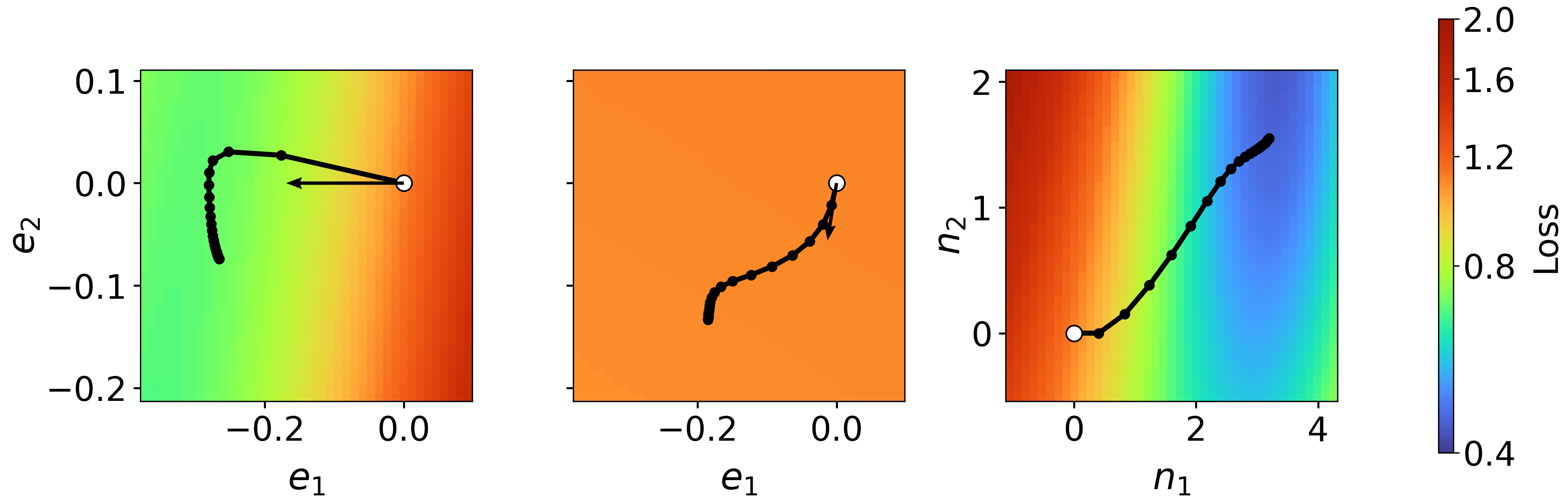}
  \caption{Activation bias redirects parameter updates, producing a distinct
trajectory with a sharp, non-smooth turn. Left: the
unshifted student on the initial-gradient plane. Center: the shifted student on
the same plane. Right: the shifted student on the complementary plane.
Colors closer to red indicate higher loss, while colors closer to blue indicate lower loss.}
  \label{fig:toy-landscape}
\end{figure}

\subsection{Activation Bias Steepens the Loss Landscape}
\label{sec:toy-hessian-sharpness}

Although the projected landscapes suggest differences in curvature, each shows
only a two-dimensional slice. We therefore quantify local curvature in
the full parameter space using the leading eigenvalues of the Hessian
$\hessian=\nabla^2\loss$, whose largest eigenvalue is commonly termed
sharpness. Figure~\ref{fig:toy-eigenvalues} shows that the shifted student has
substantially smaller eigenvalues throughout training. Thus, in this controlled
model, activation bias is associated with a markedly steeper local loss
landscape, whereas reducing the bias yields a flatter one.

\begin{figure}[t]
  \centering
  \includegraphics[width=\linewidth]{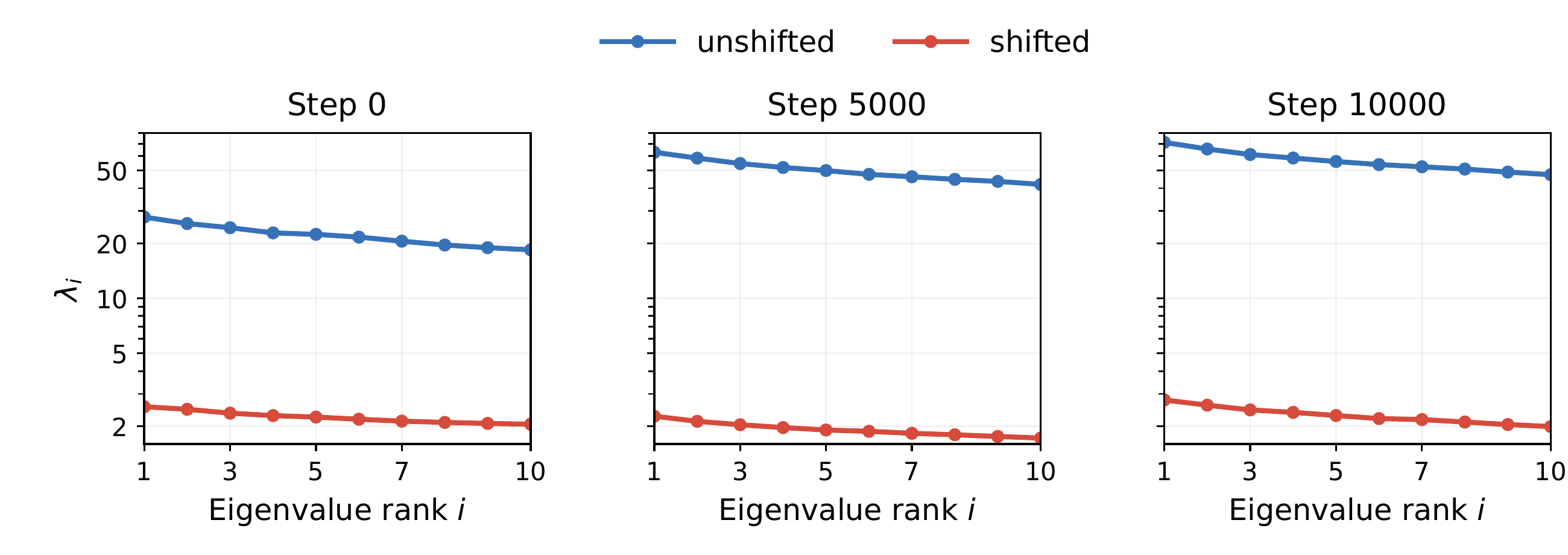}
  \caption{Reducing activation bias substantially lowers Hessian eigenvalues
throughout training, yielding a flatter local loss landscape.}
  \label{fig:toy-eigenvalues}
\end{figure}

To examine how activation bias affects loss landscape curvature, we consider a
simplified case with a post-activation vector $\mathbf a(x)$ and a trainable down-projection
$\mathbf W_2$ with squared loss
\begin{equation}
\mathcal L_b(\mathbf W_2)
=\frac12\mathbb E_x
\left\|\mathbf W_2\bigl(\mathbf a(x)-b\mathbf1\bigr)-\mathbf y(x)\right\|^2.
\label{eq:main-down-projection-loss}
\end{equation}
Suppose $\mathbb E[\mathbf a]=\mu\mathbf1$, and let
$\Sigma=\operatorname{Cov}(\mathbf a)$. The Hessian $H(b)$ with respect to
$\mathbf W_2$ then satisfies
\begin{equation}
H(b)
=\left[\Sigma+(\mu-b)^2\mathbf1\mathbf1^\top\right]\otimes I
\succeq \Sigma\otimes I,
\label{eq:main-mean-curvature}
\end{equation}
where $I$ is the identity matrix and $\otimes$ denotes the Kronecker
product. Since $(\mu-b)^2\mathbf1\mathbf1^\top\otimes I$ is positive
semidefinite, the largest eigenvalue satisfies
$\lambda_{\max}(H(b))\geq\lambda_{\max}(H(\mu))$. Thus, $b=\mu$ minimizes
Hessian sharpness in this setting.
A full proof and an extension to a more general multi-layer model are provided
in Appendix~\ref{app:hessian}.

Taken together, these measurements show that activation bias gives the
unshifted student a clear early learning advantage, but this advantage
coincides with an overshooting hidden-state trajectory and a markedly steeper
loss landscape. Reducing the bias slows initial learning, changes the
parameter-update directions, and produces smoother dynamics with stronger
long-term performance. Additional ablations are provided in
Appendix~\ref{app:toy-model-ablations}.

\section{ Analysis of Activation Bias Component Information}
\label{sec:feature-learning-analysis}

Section~\ref{sec:landscape-analysis} shows that
activation bias yields a transient early advantage yet worse long-run
optimization. We next ask whether the associated activation bias component
carries task-relevant information that the model uses.

\begin{figure}[t]
  \centering
  \includegraphics[width=\linewidth]{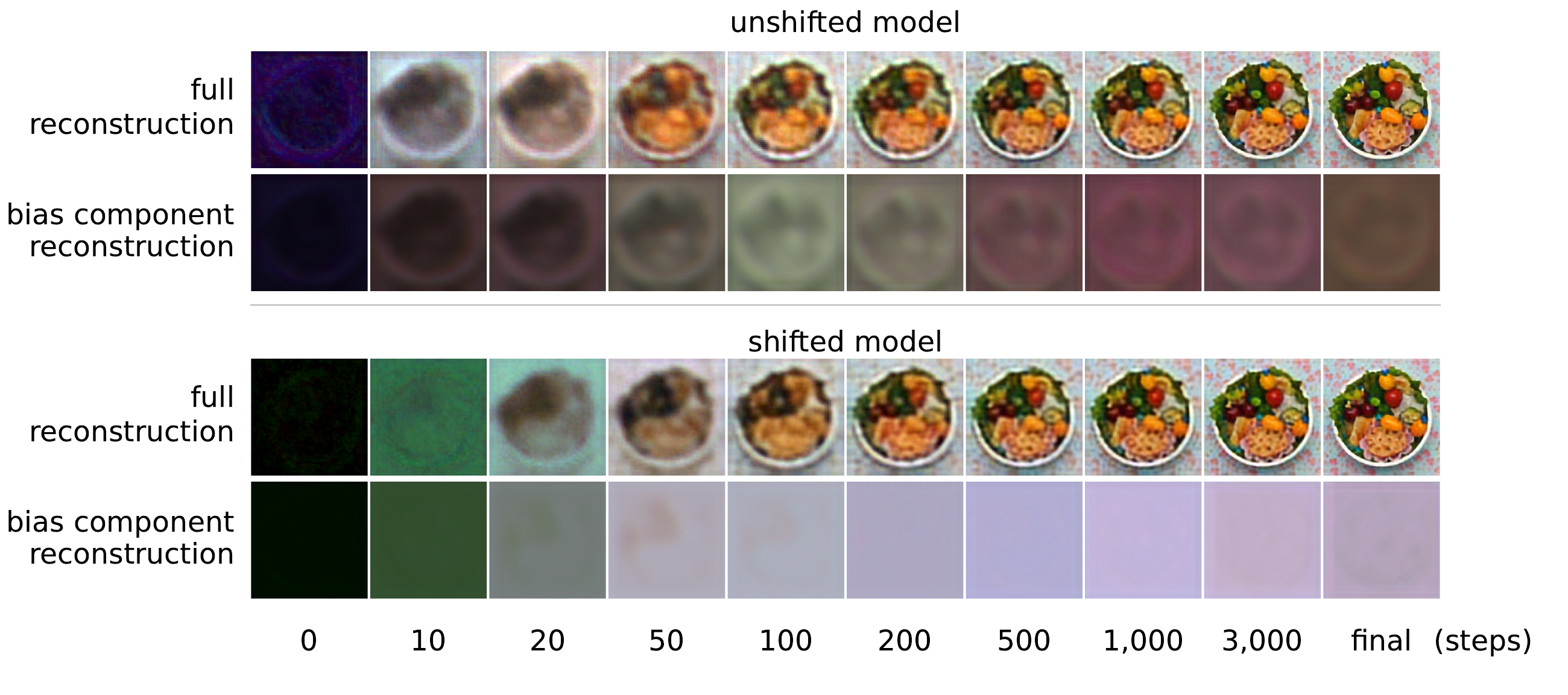}
  \caption{The unshifted model recovers recognizable image content from the activation
  bias component early in training but relies less on it later; the shifted
  model recovers little content from this component.}
  \label{fig:vgg-sample}
\end{figure}

\subsection{Image Reconstruction Setup}
\label{sec:feature-learning-setup}

Image reconstruction makes information carried by the activation bias component directly
observable in decoded images. We therefore choose this task and
use a simplified VGG-style autoencoder,
\[
  \mathbf{x}\ \xrightarrow{\ \mathrm{encoder}\ }\ \mathbf{h}
  \ \xrightarrow{\ \mathrm{decoder}\ }\ \widehat{\mathbf{x}},
\]
where $\mathbf{x}$ is an input image, $\mathbf{h}$ is its latent
representation, and $\widehat{\mathbf{x}}$ is the reconstructed image.
We use part of the VGG-19 architecture~\citep{simonyan2015very} as the encoder,
followed by a convolutional decoder. The latent representation is the encoder's
final post-activation output, with 512 channels at each spatial location.
The detailed architecture is provided in Appendix~\ref{app:vgg-reconstruction}.
We compare otherwise identical models using the unshifted and shifted activations.
Consistent with the controlled teacher--student model in
Section~\ref{sec:toy-model}, the unshifted model learns faster initially,
but the shifted model slightly surpasses it later in training, as shown in
Appendix~\ref{app:vgg-reconstruction}.

Throughout training, both decoders receive the full representation
$\mathbf{h}$. For post-hoc analysis at selected checkpoints, let
$\mathbf{h}_p\in\mathbb{R}^m$ denote the hidden state at spatial location
$p$, where $m=512$ is the number of channels. For every location $p$, we decompose it as
\begin{equation}
  \mathbf{h}_p
  = \underbrace{\left(\frac{1}{m}\mathbf{1}^{\top}\mathbf{h}_p\right)
  \mathbf{1}}_{\mathbf{h}_{p,\mathrm{bias}}}
  + \underbrace{\left(\mathbf{h}_p-\mathbf{h}_{p,\mathrm{bias}}\right)}_
  {\mathbf{h}_{p,\mathrm{dev}}}.
  \label{eq:bias-component-decomposition}
\end{equation}
\subsection{Visualization and Analysis of Activation-Bias-Component Information}
\label{sec:bias-component-analysis}

For each selected checkpoint, we pass the two components separately through the
decoder. Figure~\ref{fig:vgg-sample} shows the resulting
component-wise reconstructions. For the unshifted model, reconstruction from
$\mathbf{h}_{\mathrm{bias}}$ alone yields a recognizable, though incomplete,
image that retains coarse structure, color, and texture, implying that the model uses
some information encoded in the activation bias component. By contrast, the
corresponding reconstruction for the shifted model contains little recoverable
structure, suggesting that this model makes less use of the component.

The role of the unshifted model's activation bias component evolves non-monotonically.
Its isolated reconstruction becomes increasingly similar to the full
reconstruction, reaching its closest agreement early in training. It subsequently diverges
from the full reconstruction. This rise-and-fall
pattern mirrors the early learning advantage and hidden-state overshoot in
Section~\ref{sec:hidden-state-overshoot}.

Together, these visualizations show that activation bias can change how the model
uses image information during training. The unshifted model's transient use of
the activation bias component provides a feature-learning counterpart to its
early learning advantage, whereas the shifted model makes little use of this
component.

\section{Scale-Up Experiments}
\label{sec:scale-up}

We next examine whether our findings extend to larger, more complex
architectures and diverse tasks. We consider supervised image classification,
diffusion-based image generation, and language modeling. Unlike in the
controlled model, subtracting the same fixed scalar in every layer is
insufficient because activation bias differs across layers and evolves rapidly during
training. We therefore use \emph{adaptive shifting} as the bias-reduced
condition. This layerwise intervention subtracts the current post-activation
mean from every neuron in each layer:
\begin{equation}
  \widetilde{\phi}^{(\ell)}\!\left(\mathbf{z}^{(\ell)}\right)
  = \phi\!\left(\mathbf{z}^{(\ell)}\right)
  - \frac{1}{N_\ell}\sum_{i=1}^{N_\ell}
  \phi\!\left(z_i^{(\ell)}\right)\mathbf{1},
  \label{eq:adaptive-shifting}
\end{equation}
where $\mathbf{z}^{(\ell)}\in\mathbb{R}^{N_\ell}$ is the pre-activation vector
in layer $\ell$, $N_\ell$ is the number of neurons, and $i$ indexes neurons.
The shifted output has zero mean across neurons by construction.

\subsection{Vision Transformer Image Classification}
\label{sec:scale-up-vit}

We first consider Vision Transformers for supervised image classification. Following the ImageNet-1K
training protocol of
\citet{zhu2025transformers}, we train ViT-Base
\citep{dosovitskiy2021image} for $300$ epochs. We hold all other training choices
fixed across conditions, varying only whether adaptive shifting is applied.
Representative ViT-Base results are shown in
Figure~\ref{fig:vitb-residuals}. The model with adaptive shifting initially trails the
unshifted baseline by a substantial margin, but the ordering reverses during
training: by the end, it attains both lower test loss and higher top-1
classification accuracy, despite retaining a higher training loss. This result
is consistent with our controlled-model observation that reduced activation
bias is associated with lower Hessian sharpness and a flatter local landscape,
which is often linked to better generalization and less overfitting.

Figure~\ref{fig:scaleup-dynamics}(a) tracks the residual-stream hidden-state
trajectory using the statistic defined in Section~\ref{sec:hidden-state-overshoot}.
The unshifted model
exhibits a pronounced early overshoot during the period in which it holds a
substantial performance advantage. Under this larger architecture and more
complex task, both trajectories are noisy, but the adaptively shifted model
exhibits a substantially smaller overshoot. This contrast echoes the pattern
observed in the controlled model. Further examples of overshoot
trajectories in scale-up language models are provided in
Appendix~\ref{app:scaleup-overshoot}. Together, these results show that
reducing activation bias remains beneficial for image classification at larger
scale.

\begin{figure}[t]
  \centering
  \includegraphics[width=\linewidth]{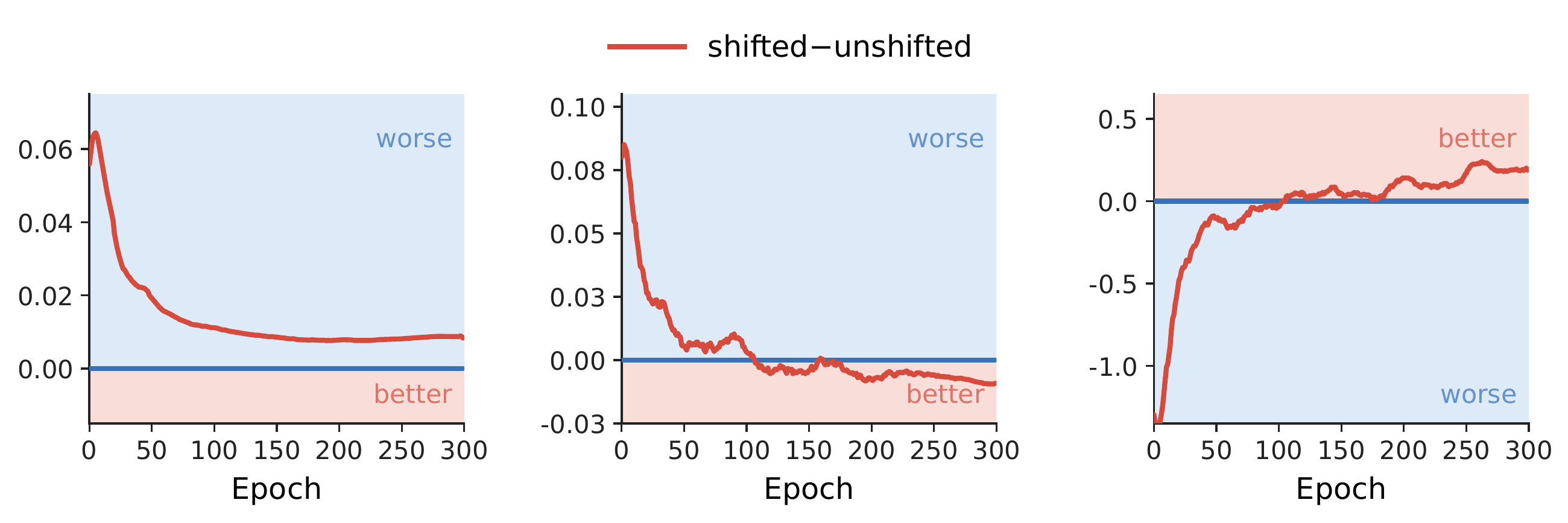}
  \vspace{-2.5mm}

  \makebox[\linewidth][l]{%
    \hspace{0.171\linewidth}\makebox[0pt][c]{\small (a) Training-loss residual}%
    \hspace{0.331\linewidth}\makebox[0pt][c]{\small (b) Test-loss residual}%
    \hspace{0.331\linewidth}\makebox[0pt][c]{\small (c) Top-1 accuracy residual}%
  }
  \caption{For the ViT image-classification task, the shifted model retains a
  higher training loss but ultimately achieves lower test loss and higher test
  accuracy. The loss
  residuals are $L_{\mathrm{shifted}}-L_{\mathrm{unshifted}}$, and the accuracy
  residual is $\mathrm{Acc}_{\mathrm{shifted}}-\mathrm{Acc}_{\mathrm{unshifted}}$.}
  \label{fig:vitb-residuals}
\end{figure}

\begin{figure}[t]
  \centering
  \includegraphics[width=\linewidth]{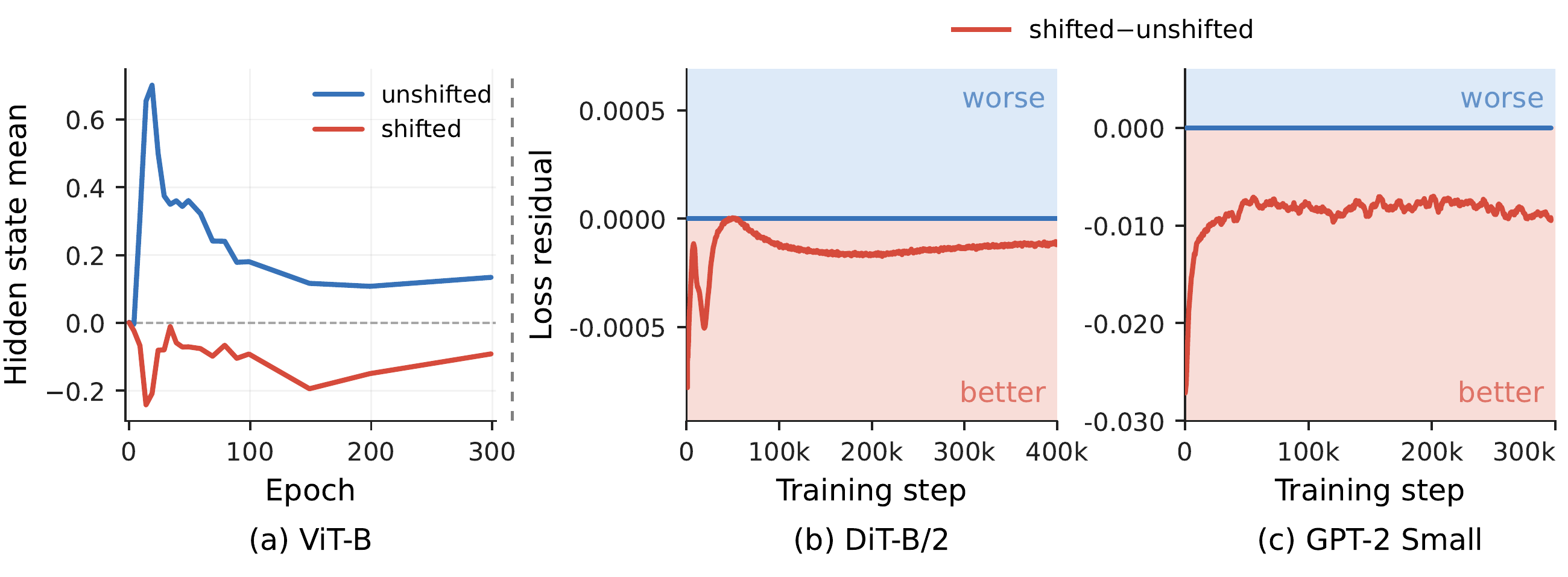}
  \caption{(a) Reducing activation bias suppresses hidden-state overshoot in
  ViT. (b) and (c): The shifted model performs better in both DiT image
  generation and GPT-2 language pretraining. The loss residual is
  $L_{\mathrm{shifted}}-L_{\mathrm{unshifted}}$.}
  \label{fig:scaleup-dynamics}
  \label{fig:gpt2-residual}
\end{figure}

\subsection{Diffusion Transformer Image Generation}
\label{sec:scale-up-dit}

We next consider image generation with Diffusion Transformers
(DiTs)~\citep{peebles2023dit}. Following the official
\texttt{facebookresearch/DiT} implementation, we train DiT-B/2 on ImageNet-256
for $400{,}000$ optimization steps. We hold all other training choices fixed and vary only
whether adaptive shifting is applied. As shown in
Figure~\ref{fig:scaleup-dynamics}(b), adaptive shifting achieves comparable or
lower training loss throughout training. Additional results in
Appendix~\ref{app:additional-scaleup} show that the shifted model achieves lower FID at every evaluated
checkpoint. These results demonstrate that reducing activation bias can also
benefit larger-scale image generation.
\subsection{GPT-2 Pretraining}
\label{sec:scale-up-gpt2}

We then move from vision to language modeling. Following Karpathy's training
setup~\citep{karpathy2022nanogpt}, we train GPT-2 Small~\citep{radford2019language}
from scratch on the full OpenWebText corpus~\citep{gokaslan2020openwebtext}
for $300{,}000$ optimization steps. As above, we
hold all other training choices fixed and vary only whether adaptive shifting
is applied.
Figure~\ref{fig:scaleup-dynamics}(c) reports the results. Adaptive shifting
consistently improves the performance, stably lowering loss by about 0.01 in
absolute terms and perplexity by 1\% relative to the unshifted baseline.
These gains show that the effect of reducing activation bias extends to
language modeling.

The larger-scale models and tasks considered here are widely adopted
and have been extensively optimized, leaving relatively little headroom for
improvement from a single change to the activation. Even so, adaptive shifting
yields small but consistent gains. Together with the observed changes in
learning dynamics, these gains indicate that activation bias has a
non-negligible effect in practical models. Additional scale-up results are
provided in Appendix~\ref{app:additional-scaleup}.

\section{Discussions}
\label{sec:discussion}

The actual effect of activation bias may vary in different settings. The extent
of early acceleration, hidden-state overshoot, and final performance differences
may depend on data and task structure, model scale, architectural components
such as attention, and weight initialization. For example, during GPT-2 Small
pretraining, the unshifted model does not exhibit the initial performance
advantage seen in the controlled model, potentially due to the small
initialization scale used by this architecture. Nevertheless, across the
settings studied here, smaller activation bias consistently yields better
late-stage performance. How these factors interact with activation bias merits
further investigation.

We use fixed or adaptive shifts to isolate a learning-dynamical consequence of
activation bias. Although reducing the bias improves performance in our
experiments, these results neither establish that zero post-activation bias
is universally optimal nor guarantee that either intervention will benefit every
architecture, task, or training regime. Performance can also depend on other
desirable activation properties, including selectivity, signal scale, and
smoothness. The methods used here to control activation bias therefore serve
as diagnostic interventions instead of general prescriptions for activation
design.

Our results identify a set of connected signatures of activation bias: an
early--late learning pattern, hidden-state overshoot and non-smooth dynamics,
distinct parameter-update directions, a steeper loss landscape, and transient
use of information in the activation bias component. We do not yet have a
unified mathematical theory or a complete causal account of how these effects
arise. Moreover, our larger-model experiments do not systematically test
whether all the phenomena and theoretical results from the controlled setting
extend to more complex architectures. Addressing these gaps could provide
more principled guidance for the design of new activation functions.

Activation functions have many properties beyond bias that we do not study
here. Post-activation variance is one example: the variance most conducive to
learning may depend on architecture, task, scale, and training regime, yet its
effects on learning dynamics and final performance remain unclear. Future work
could isolate variance and other properties of activation functions under controlled conditions, track their effects throughout
training, and test whether the resulting patterns transfer across models and
tasks. Such analyses would provide a broader foundation for activation-function
design.

\section{Related Work}
\label{sec:related-work}

Activation-function design spans nonlinearities with markedly different
input--output profiles. Early work used sigmoid and tanh
\citep{lecun1998efficient}; later proposals introduced Softplus
\citep{dugas2000incorporating}, ReLU~\citep{nair2010relu},
GELU~\citep{hendrycks2016gelu}, and Swish~\citep{ramachandran2017swish}.
Leaky ReLU~\citep{maas2013rectifier}, PReLU~\citep{he2015delving}, and
Mish~\citep{misra2020mish} further broadened this design space. More recent
work has explored gated activations such as SwiGLU~\citep{shazeer2020glu},
learnable activations such as ACON~\citep{ma2021activate}, and alternatives
such as SMU~\citep{biswas2022smooth} and
StarReLU~\citep{yu2024metaformer}. These activation functions have proved effective across tasks, yet
many of them produce predominantly positive outputs and thus introduce activation
bias. While prior work has focused on designing and evaluating the functions
themselves, the effect of this bias on learning remains less understood. Our
work complements these studies by isolating activation bias and examining how
it shapes learning dynamics.

Previous work on activation functions and their statistics has primarily
sought to improve trainability. To stabilize intermediate signals,
normalization methods center or rescale
them~\citep{ioffe2015batch,ba2016layer,arpit2016normalization,bjorck2018understanding};
mean-only batch normalization retains centering alone, whereas RMSNorm
retains rescaling alone~\citep{salimans2016weight,zhang2019root}.
\citet{santurkar2018how} showed that batch normalization also smooths the
loss landscape. To sustain signal and gradient flow through depth,
mean-field and dynamical-isometry studies analyze activation
choices~\citep{schoenholz2017deep,pennington2017resurrecting,hayou2019impact,lesser2026competing},
and \citet{qi2020deep} designed a shifted ReLU to improve isometry. To
counter shifts associated with nonzero activation outputs,
\citet{raiko2012deep} centered neuron outputs and slopes to improve gradient
conditioning, \citet{eidnes2018shifting} introduced bipolar activations to
facilitate deep training, and \citet{kutsuna2019linearly} constrained weights
to reduce shifts in downstream pre-activations. These studies alter
normalization, activation shape, or weights, whereas we isolate activation
bias as a single factor and investigate its effects on learning dynamics.

Some studies have examined nonzero activation bias.
\citet{glorot2010understanding}
linked the positive mean of sigmoid activations to upper-layer saturation at
initialization. \citet{clevert2016fast} showed that ELUs can reduce bias shift,
accelerate learning, and improve classification accuracy;
\citet{klambauer2017self} designed SELUs to stabilize activation means and
variances across depth; and \citet{brock2021characterizing} controlled the
growth of positive channel means in unnormalized ResNets. These studies
establish consequences of activation bias in restricted settings. We examine
its effects throughout training and show that reducing it can improve
model performance across different architectures and tasks.
\section{Conclusion}
\label{sec:conclusion}

We study how activation bias shapes training dynamics and the use of
task-relevant information. In the controlled setting, activation bias confers
an early learning advantage but is associated with hidden-state overshoot,
non-smooth trajectories, a steeper loss landscape and worse late-stage performance. Our reconstruction analysis further
shows that the learned model transiently uses information in the bias
component. Across larger and more diverse settings, adaptive shifting retains a
late-stage performance advantage. These results motivate a more complete causal
account of activation bias and its explicit consideration in future
activation-function design.

\section*{AI Use Statement}

We used generative AI tools to brainstorm ideas, formulate candidate
mathematical claims, draft code following explicit author instructions, and
polish the writing. These tools were not used to create
synthetic datasets, perform data analysis, or write mathematical proofs.
The authors critically reviewed all AI-assisted ideas and text,
independently inspected and verified generated code, and conducted the data
analysis and proofs themselves. The authors take full responsibility for the
final content of this paper.

\bibliography{references}
\bibliographystyle{plainnat}

\clearpage
\appendix

\section{Estimating $b$ for Different Activation Functions}
\label{app:optimal-b}

This appendix examines fixed, layer-shared activation shifts in the controlled
model. We first derive the shift $b_{\mathrm{theory}}$ that brings the
post-activation mean near zero at initialization and show that the empirical
means remain close to zero during training. We then sweep $b$ and measure the
final loss. Across ReLU, GELU, Swish, Tanh, and sigmoid, the
theory-predicted shifts achieve near-optimal final loss.

\subsection{Theory-predicted activation means}
\label{app:optimal-b-theory}

For a pre-norm MLP block, let $\mathbf{x}\in\mathbb{R}^d$ be its input and
$\mathbf{u}=\operatorname{LN}(\mathbf{x})$. Let $\mathbf{w}_i$ be the $i$th row of the
up-projection linear layer. The pre-activation of neuron $i$ is
\begin{equation}
  Z_i = \mathbf{w}_i^\top \mathbf{u}.
  \label{eq:preactivation-projection}
\end{equation}
The weights are initialized independently of $\mathbf{u}$ as
\begin{equation}
  w_{ij}\overset{\mathrm{i.i.d.}}{\sim}
  \mathcal{U}\!\left[-\frac{1}{\sqrt{d}},
  \frac{1}{\sqrt{d}}\right].
\end{equation}
Using $\|\mathbf{u}\|_2^2\approx d$ from LayerNorm gives
\begin{equation}
  \mathbb{E}_{\mathbf{w}_i,\mathbf{x}}[Z_i]=0,
  \qquad
  \mathbb{E}_{\mathbf{w}_i,\mathbf{x}}[Z_i^2]
  = \frac{1}{3d}\mathbb{E}_{\mathbf{x}}\|\mathbf{u}\|_2^2
  \approx \frac{1}{3}.
\end{equation}
We approximate the pre-activations by
\begin{equation}
  Z\sim\mathcal{N}\!\left(0,\frac{1}{3}\right).
  \label{eq:gaussian-preactivation-approximation}
\end{equation}
The shifted activation is
\begin{equation}
  \phi_{\mathrm{shifted}}(Z;b)=\phi(Z)-b.
\end{equation}
Thus,
\begin{equation}
  b_{\mathrm{theory}}(\phi)
  :=\mathbb{E}[\phi(Z)].
  \label{eq:b-theory-definition}
\end{equation}

\paragraph{ReLU.}
For $\phi(z)=\max\{0,z\}$,
\begin{align}
  b_{\mathrm{theory}}(\operatorname{ReLU})
  &= \int_{-\infty}^{\infty}\max\{0,z\}
  \frac{1}{\sqrt{2\pi/3}}
  \exp\!\left(-\frac{z^2}{2/3}\right)\,\mathrm{d}z \notag\\
  &= \sqrt{\frac{3}{2\pi}}\int_0^\infty
  z\exp\!\left(-\frac{3z^2}{2}\right)\,\mathrm{d}z \notag\\
  &= \frac{1}{\sqrt{6\pi}} \notag\\
  &\approx 0.230.
  \label{eq:relu-theory-shift}
\end{align}

\paragraph{GELU.}
We use the exact GELU, $\phi(z)=z\Phi(z)$, where $\Phi$ and $\varphi$
are the standard Gaussian cumulative distribution and density functions.
\begin{align}
  b_{\mathrm{theory}}(\operatorname{GELU})
  &= \int_{-\infty}^{\infty}z\Phi(z)
  \frac{1}{\sqrt{2\pi/3}}
  \exp\!\left(-\frac{z^2}{2/3}\right)\,\mathrm{d}z \notag\\
  &= \frac{1}{\sqrt{6\pi}}\int_{-\infty}^{\infty}
  \varphi(z)\exp\!\left(-\frac{3z^2}{2}\right)\,\mathrm{d}z \notag\\
  &= \frac{1}{\sqrt{24\pi}} \approx 0.115.
  \label{eq:gelu-theory-shift}
\end{align}

\paragraph{Swish.}
For $\phi(z)=z/(1+e^{-z})$,
\begin{align}
  b_{\mathrm{theory}}(\operatorname{Swish})
  &= \int_{-\infty}^{\infty}
  \frac{z}{1+e^{-z}}
  \frac{1}{\sqrt{2\pi/3}}
  \exp\!\left(-\frac{z^2}{2/3}\right)\,\mathrm{d}z \notag\\
  &\approx 0.077.
  \label{eq:silu-theory-shift}
\end{align}
\paragraph{Tanh.}
Since $\tanh$ is odd,
\begin{equation}
  b_{\mathrm{theory}}(\tanh)
  = \mathbb{E}[\tanh(Z)] = 0.
  \label{eq:tanh-theory-shift}
\end{equation}

\paragraph{Sigmoid.}
For $\phi(z)=1/(1+e^{-z})$, we have $\phi(z)+\phi(-z)=1$.
Since $Z$ and $-Z$ have the same distribution,
\begin{align}
  b_{\mathrm{theory}}(\operatorname{sigmoid})
  &= \mathbb{E}[\phi(Z)] \notag\\
  &= \frac12\mathbb{E}[\phi(Z)+\phi(-Z)] \notag\\
  &= \frac12 = 0.500.
  \label{eq:sigmoid-theory-shift}
\end{align}

Figure~\ref{fig:activation-mean-dynamics} shows that the predicted shifts
keep the post-activation means near zero during training.

\begin{figure}[!htbp]
  \centering
  \includegraphics[width=0.95\linewidth]{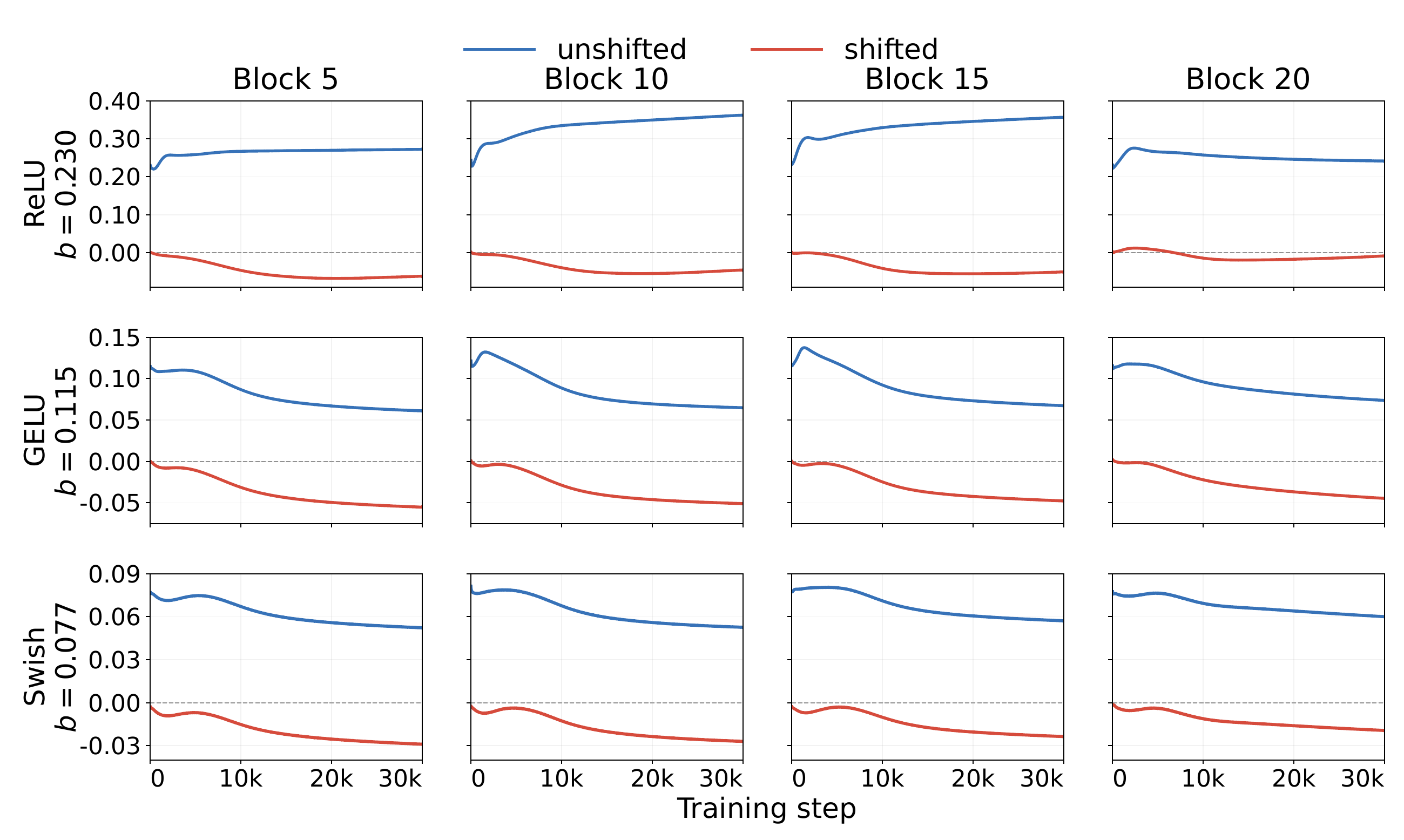}
  \caption{The post-activation means of
  the shifted models remain near zero throughout training.}
  \label{fig:activation-mean-dynamics}
\end{figure}

\subsection{Final loss as a function of the activation shift}
\label{app:optimal-b-experiment}

Using the same teacher--student training procedure, we sweep the shift $b$
for each activation over multiple random seeds. Within each seed, all compared
shifts use the same teacher, student initialization, and sequence of training
batches.
Figure~\ref{fig:activation-final-loss-vs-b} plots final loss
against $b$. Across the tested activations, the theory-predicted shifts achieve
near-optimal final loss.

\begin{figure}[!htbp]
  \centering
  \includegraphics[width=\linewidth]{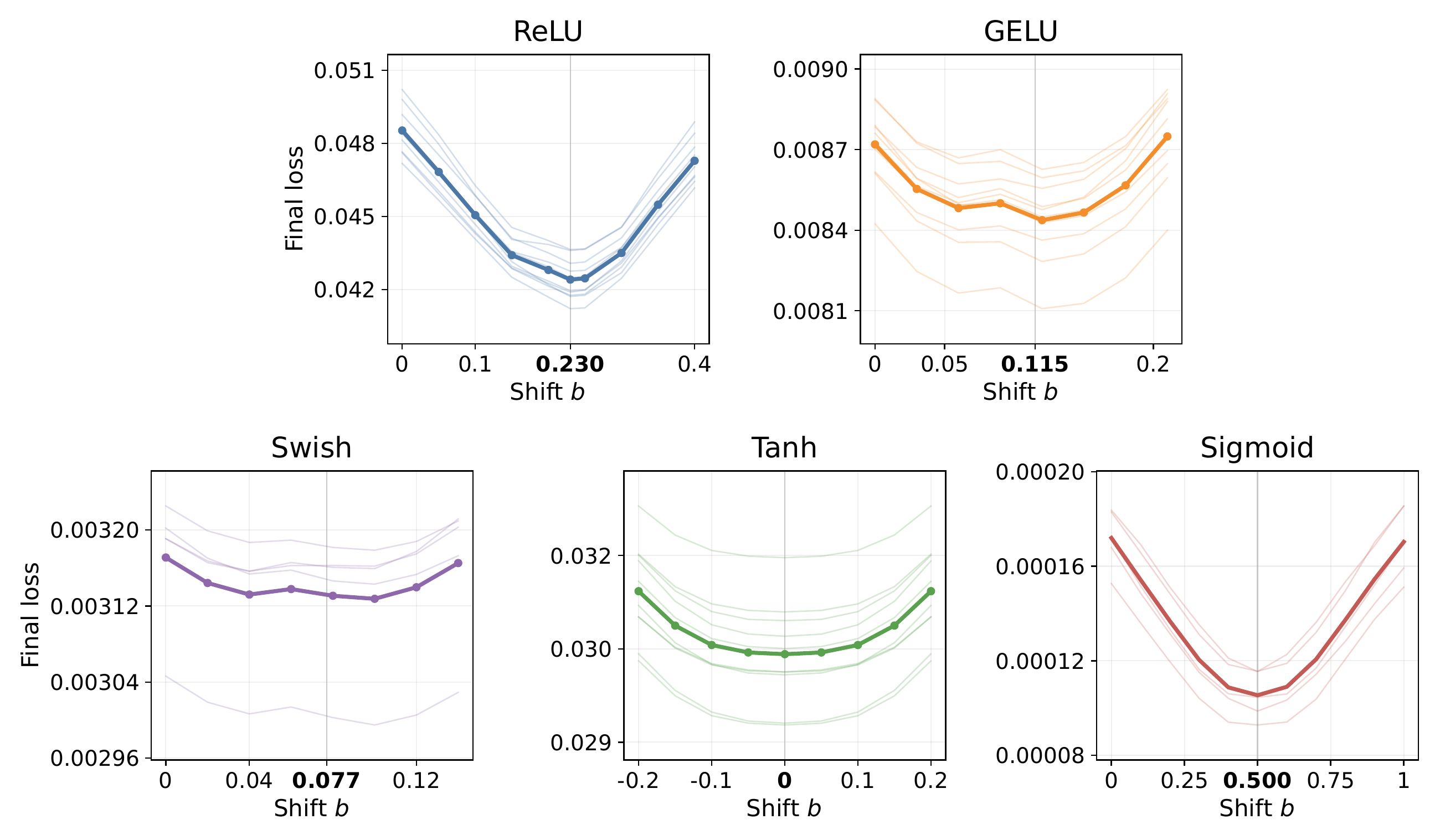}
  \caption{Across all five activations, our theory's predicted shifts achieve
  near-optimal final loss. Thin curves show individual seeds, and thick curves
  show their mean.}
  \label{fig:activation-final-loss-vs-b}
\end{figure}

\section{Theoretical analysis of optimization landscape}
\label{app:hessian}

This appendix analyzes how activation bias affects loss landscape curvature
and changes the direction of parameter updates early in
training.

\subsection{Activation mean and landscape sharpness}
\label{app:sharpness-theory}

\subsubsection{A single trainable down-projection}

We start with the simplest setting, where $a(x)\in\mathbb{R}^{m}$ is fixed and
only a down-projection layer $B\in\mathbb{R}^{d_{\mathrm{out}}\times m}$ is trained.
Define
\begin{equation}
  z_b(x):=a(x)-b\mathbf{1},
  \qquad
  f_B(x)=Bz_b(x).
  \label{eq:single-down-model}
\end{equation}
For target $y(x)$, define
\begin{align}
  \mathcal{L}_b(B)
  &=\frac{1}{2}\mathbb{E}_x\left\|Bz_b(x)-y(x)\right\|^2,
  \label{eq:single-down-loss}\\
  M_b&:=\mathbb{E}_x\left[z_b(x)z_b(x)^\top\right].
  \label{eq:feature-second-moment}
\end{align}

\begin{proposition}[Exact down-projection sharpness]
\label{prop:single-down-sharpness}
The Hessian with respect to the vectorized matrix $B$ is
\begin{equation}
  H_B(b)=M_b\otimes I_{d_{\mathrm{out}}},
  \qquad
  \lambda_B(b):=\lambda_{\max}\!\left(H_B(b)\right)
  =\lambda_{\max}(M_b).
  \label{eq:single-down-hessian}
\end{equation}
Here $\otimes$ denotes the Kronecker product.
Suppose
\begin{equation}
  \mathbb{E}[a]=\mu\mathbf{1},
  \label{eq:shared-neuron-mean}
\end{equation}
so all neurons have the same mean activation $\mu$. Then $b=\mu$ globally
minimizes $\lambda_B(b)$ over all scalar shifts.
\end{proposition}

\begin{proof}
For any perturbation $\Delta B$,
\begin{equation}
  \left.\frac{\mathrm{d}^2}{\mathrm{d}t^2}
  \mathcal{L}_b(B+t\Delta B)\right|_{t=0}
  =\mathbb{E}\|\Delta Bz_b\|^2
  =\operatorname{tr}(\Delta B M_b\Delta B^\top),
\end{equation}
which gives \eqref{eq:single-down-hessian}. Let
$\Sigma=\operatorname{Cov}(a)$. Then
\begin{equation}
  M_b
  =\Sigma+(\mu-b)^2\mathbf{1}\mathbf{1}^\top
  \succeq \Sigma=M_\mu.
  \label{eq:rank-one-mean-curvature}
\end{equation}
\begin{equation*}
  \lambda_B(b)
  =\lambda_{\max}(M_b)
  \ge\lambda_{\max}(M_\mu)
  =\lambda_B(\mu). \qedhere
\end{equation*}
\end{proof}

When neuron-wise activation means differ, write
\begin{equation*}
  \mathbb{E}[a]=\bar\mu\mathbf{1}+\varepsilon,
  \qquad
  \bar\mu:=\frac{1}{m}\mathbf{1}^\top\mathbb{E}[a],
  \qquad \varepsilon\perp\mathbf{1}.
\end{equation*}
Since $M_{\bar\mu}=\Sigma+\varepsilon\varepsilon^\top$ and
$M_b\succeq\Sigma$ for every $b$,
\begin{equation}
  0\le
  \lambda_{\max}(M_{\bar\mu})-\min_b\lambda_{\max}(M_b)
  \le \|\varepsilon\|^2.
  \label{eq:unequal-mean-robustness}
\end{equation}

Thus shifting by the average activation mean is within
$\|\varepsilon\|^2$ of the optimal scalar shift in largest-eigenvalue
curvature. The gap is small whenever $\|\varepsilon\|$ is small.

\subsubsection{Jointly training the up- and down-projections}

We then allow both the up-projection $A$ and the down-projection $B$ to be
trained. At a fixed parameter state, let
\begin{equation}
  a(x)=\phi(Ax),\qquad z_b(x)=a(x)-b\mathbf{1},
  \qquad f_{A,B}(x)=Bz_b(x),
\end{equation}
and assume $\mathbb{E}[a]=\mu\mathbf{1}$. Let
$D_x=\operatorname{diag}(\phi'(Ax))$. To first order,
\begin{equation}
  \Delta f(x)=BD_x\Delta A x+\Delta Bz_b(x).
  \label{eq:joint-first-order-output}
\end{equation}
Let $J=[J_A\;J_B]$ be the output Jacobian with respect to
$\theta=(\operatorname{vec}(A),\operatorname{vec}(B))$. Under column-wise
vectorization, \eqref{eq:joint-first-order-output} gives
\begin{equation}
  J_A(x)=x^\top\otimes(BD_x),
  \qquad
  J_B(x)=z_b(x)^\top\otimes I_{d_{\mathrm{out}}}.
  \label{eq:joint-jacobian-blocks}
\end{equation}
For squared loss, the Gauss--Newton (GN) matrix is
\begin{equation}
  G_b=\mathbb{E}[J(x)^\top J(x)]
  =
  \begin{pmatrix}
    G_{AA} & G_{AB}(b)\\
    G_{BA}(b) & G_{BB}(b)
  \end{pmatrix},
  \label{eq:joint-gn-block-form}
\end{equation}
where
\begin{align}
  G_{AA}
  &=\mathbb{E}\!\left[(xx^\top)\otimes
    (D_xB^\top B D_x)\right],
  \label{eq:joint-gn-aa}\\
  G_{AB}(b)
  &=\mathbb{E}\!\left[(x z_b(x)^\top)\otimes
    (D_xB^\top)\right],
  \qquad G_{BA}(b)=G_{AB}(b)^\top,
  \label{eq:joint-gn-ab}\\
  G_{BB}(b)
  &=\mathbb{E}\!\left[z_b(x)z_b(x)^\top\right]
    \otimes I_{d_{\mathrm{out}}}
   =M_b\otimes I_{d_{\mathrm{out}}}.
  \label{eq:gn-bb-block}
\end{align}
Thus $G_{AA}$ is independent of $b$, while the mixed blocks can depend on $b$.
Using $\lambda_B(b)$ from Proposition~\ref{prop:single-down-sharpness}, define
\begin{equation}
  \lambda_A:=\lambda_{\max}(G_{AA}),
  \qquad
  \lambda_G(b):=\lambda_{\max}(G_b).
\end{equation}

\begin{proposition}[Full-GN bound and factor-two guarantee]
\label{prop:full-gn-factor-two}
For every shift $b$,
\begin{equation}
  \max\{\lambda_A,\lambda_B(b)\}
  \le \lambda_G(b)
  \le \lambda_A+\lambda_B(b).
  \label{eq:full-gn-sandwich}
\end{equation}
Let $b^*\in\arg\min_b\lambda_G(b)$. Consequently,
\begin{equation}
  \lambda_G(b^*)\le\lambda_G(\mu)\le 2\lambda_G(b^*).
  \label{eq:full-gn-factor-two}
\end{equation}
\end{proposition}

\begin{proof}
Restricting a parameter direction to the $A$ or $B$ block gives the lower
bound. For any perturbation satisfying
$\|\Delta A\|_F^2+\|\Delta B\|_F^2=1$, the triangle inequality and
Cauchy--Schwarz give
\begin{align}
  \mathbb{E}\|J_A\operatorname{vec}(\Delta A)
    +J_B\operatorname{vec}(\Delta B)\|^2
  &\le
  \left(\sqrt{\lambda_A}\|\Delta A\|_F
  +\sqrt{\lambda_B(b)}\|\Delta B\|_F\right)^2 \\
  &\le \lambda_A+\lambda_B(b),
\end{align}
which proves \eqref{eq:full-gn-sandwich}. Since $b=\mu$ minimizes
$\lambda_B(b)$,
\begin{equation}
  \lambda_G(b^*)
  \le \lambda_G(\mu)
  \le \lambda_A+\lambda_B(\mu)
  \le \lambda_A+\lambda_B(b^*)
  \le 2\lambda_G(b^*).
\end{equation}
\end{proof}

We test these bounds in the one-block model across ten seeds. At each training
checkpoint, we hold $A$ and $B$ fixed, compute the average activation mean
$\bar\mu=\mathbb{E}_{x,i}[a_i(x)]$, and sweep $b$ to compare $\lambda_G(b)$
with the bounds in \eqref{eq:full-gn-sandwich}. The minimum over the tested
shifts estimates $\lambda_G(b^*)$. In Figure~\ref{fig:one-block-gn-bounds},
the bounds closely track $\lambda_G(b)$ across shifts (left), while
$\lambda_G(\bar\mu)$ remains close to the empirical minimum throughout
training (right).

\begin{figure}[!htbp]
  \centering
  \includegraphics[width=\linewidth]{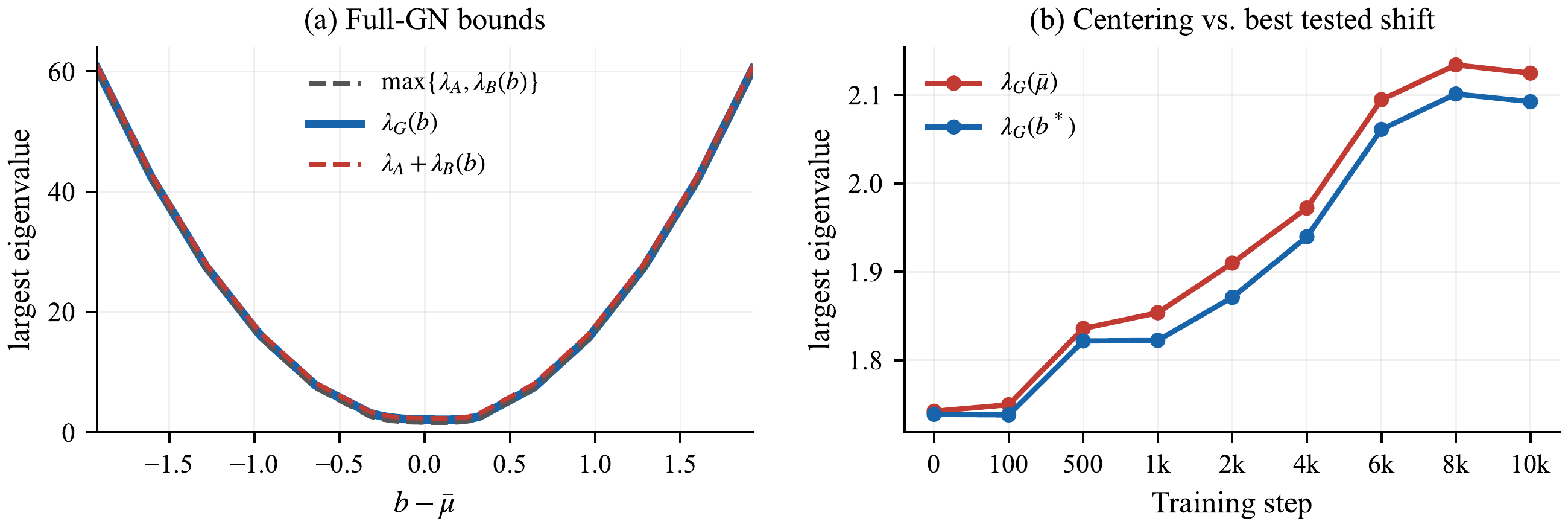}
  \caption{Empirical evaluation of Proposition~\ref{prop:full-gn-factor-two}.
  Left: The theoretical lower and upper bounds closely track the measured
  leading Gauss--Newton eigenvalue $\lambda_G(b)$.
  Right: The value at the average-mean shift, $\lambda_G(\bar\mu)$, remains
  close to the empirical minimum, $\lambda_G(b^*)$, throughout training.}
  \label{fig:one-block-gn-bounds}
\end{figure}

Thus, removing the activation bias yields nearly optimal Gauss--Newton
sharpness in this model, with a factor-two upper bound relative to the optimum.

\subsubsection{From Gauss--Newton to the true Hessian}
\label{app:gn-to-hessian}

Let $r(x)=f_\theta(x)-y(x)$ and
$g(t)=f_{\theta+t\Delta\theta}(x)$. For squared loss,
\begin{equation}
  \Delta\theta^\top H\Delta\theta
  =\underbrace{\mathbb{E}\|g'(0)\|^2}_{\Delta\theta^\top G\Delta\theta}
  +\mathbb{E}\langle r(x),g''(0)\rangle.
  \label{eq:hessian-gn-decomposition}
\end{equation}
For ReLU, consider a local parameter region in which the activation gates
$D_x=\operatorname{diag}(\mathbf{1}_{Ax>0})$ remain fixed for each input $x$. Then
\begin{equation}
  f_{A,B}(x)=B(D_xAx-b\mathbf{1}),
\end{equation}
and along a joint direction $(\Delta A,\Delta B)$,
\begin{equation}
  g''(0)=2\Delta B D_x\Delta A x.
  \label{eq:fixed-gate-second-derivative}
\end{equation}
The diagonal blocks satisfy $H_{AA}=G_{AA}$ and $H_{BB}=G_{BB}$. For a unit
joint direction,
\begin{equation}
  \left|2\mathbb{E}\langle r,\Delta B D_x\Delta A x\rangle\right|
  \le \mathbb{E}[\|r\|\,\|x\|]
  =:R_b,
\end{equation}
so
\begin{equation}
  \left|\lambda_{\max}(H_b)-\lambda_{\max}(G_b)\right|
  \le \|H_b-G_b\|_2\le R_b.
  \label{eq:hessian-gn-residual-bound}
\end{equation}
In particular,
\begin{equation*}
  \mathcal{L}_b=0\quad\Longrightarrow\quad H_b=G_b.
\end{equation*}

Figure~\ref{fig:one-block-hessian-gn-bridge} shows that the largest eigenvalues
of the Hessian and Gauss--Newton matrix at $b=\bar\mu$ remain close for most of training.

\begin{figure}[!htbp]
  \centering
  \includegraphics[width=0.68\linewidth]{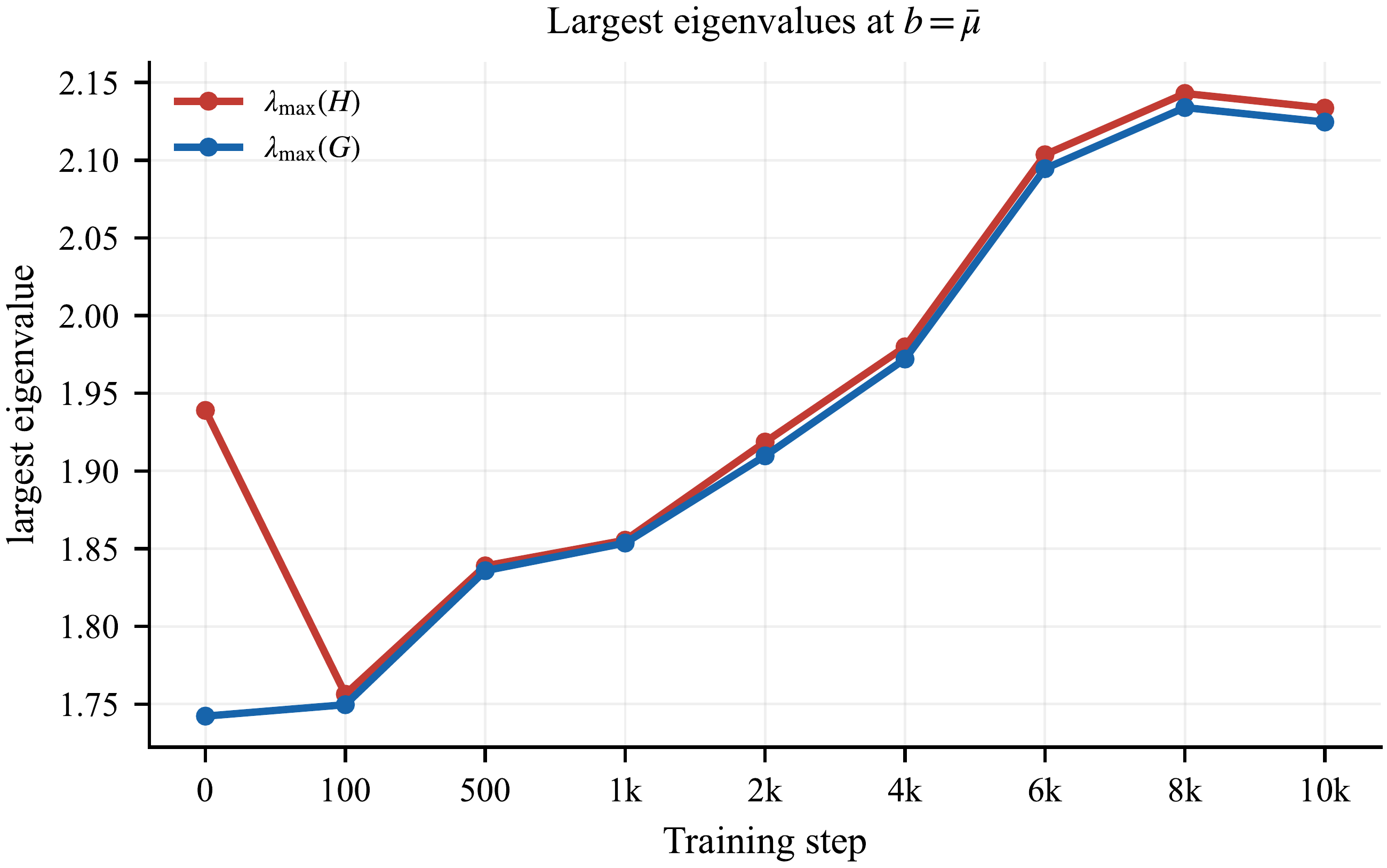}
  \caption{After 100 updates, the largest eigenvalues of the exact Hessian and
  its Gauss--Newton approximation remain close at $b=\bar\mu$.}
  \label{fig:one-block-hessian-gn-bridge}
\end{figure}

\subsubsection{Multiple down-projections with fixed activations}
\label{app:deep-sharpness-extension}

We finally consider multiple trainable down-projections with fixed activations:
\begin{equation}
  f(x)=\sum_{\ell=1}^{L}B_\ell
  \left(a_\ell(x)-b_\ell\mathbf{1}\right),
  \qquad
  \mathbb{E}[a_\ell]=\mu_\ell\mathbf{1}.
  \label{eq:frozen-multilayer-model}
\end{equation}
Define
\begin{equation}
  \Sigma_{\mathrm{stack}}
  :=\operatorname{Cov}\!\left(
  \begin{pmatrix}
    a_1(x)\\
    \vdots\\
    a_L(x)
  \end{pmatrix}
  \right),
  \qquad
  v_{\mathbf b}:=
  \begin{pmatrix}
    (\mu_1-b_1)\mathbf{1}\\
    \vdots\\
    (\mu_L-b_L)\mathbf{1}
  \end{pmatrix}.
\end{equation}
Then
\begin{equation}
  H_{\mathrm{joint}}(\mathbf b)=G_{\mathrm{joint}}(\mathbf b)
  =\left(\Sigma_{\mathrm{stack}}
  +v_{\mathbf b}v_{\mathbf b}^\top\right)\otimes I_{d_{\mathrm{out}}}.
  \label{eq:frozen-multilayer-decomposition}
\end{equation}
Thus, setting $b_\ell=\mu_\ell$ for every layer minimizes the joint sharpness.

\subsection{Activation mean and parameter gradients}
\label{app:parameter-gradients}

We next study the different update directions of unshifted and shifted models during
the early stage of training. We analyze a simplified setting
in which only a down-projection layer is trainable. Suppose
\begin{equation}
  \mathbb{E}[a(x)]=\mu\mathbf{1},
  \qquad
  \operatorname{Cov}(a(x))=qI,
  \qquad
  y(x)=B_\star a(x).
\end{equation}
The unshifted model uses $a(x)$, while the shifted model uses
\begin{equation}
  z_\mu(x)=a(x)-\mu\mathbf{1}.
\end{equation}
Both models start from $B_{\mathrm{init}}$ and use gradient descent with
learning rate $\eta$.
Let $B_{k-1}^{(0)}$ and $B_{k-1}^{(\mu)}$ denote their parameters before the
$k$th update, and define
\begin{equation}
  g_{0,k}:=\nabla_B\mathcal{L}_0\!\left(B_{k-1}^{(0)}\right),
  \qquad
  g_{\mu,k}:=\nabla_B\mathcal{L}_\mu\!\left(B_{k-1}^{(\mu)}\right).
  \label{eq:parameter-gradient-definition}
\end{equation}

Let
\begin{equation}
  e=\frac{\mathbf{1}}{\sqrt{m}}.
\end{equation}
For the unshifted and shifted models,
\begin{equation}
  M_0=qI+m\mu^2ee^\top,
  \qquad
  M_\mu=qI.
\end{equation}
Along $e$, the eigenvalue of $M_0$ is $\kappa=(q+m\mu^2)/q$ times that of
$M_\mu$.
The loss gradients satisfy
\begin{equation}
  \nabla_B\mathcal{L}_0(B)=(B-B_\star)M_0,
  \qquad
  \nabla_B\mathcal{L}_\mu(B)=(B-B_\star)M_\mu.
\end{equation}

Let
\begin{equation}
  \Delta
  =B_{\mathrm{init}}-B_\star
  =\Delta_{\parallel}+\Delta_{\perp},
\end{equation}
where
\begin{equation}
  \Delta_{\parallel}=\Delta ee^\top,
  \qquad
  \Delta_{\perp}=\Delta(I-ee^\top).
\end{equation}
The gradients used for the $k$th update satisfy
\begin{equation}
  \begin{aligned}
  g_{0,k}
  &=q(1-\eta q)^{k-1}
    \left(\Delta_{\perp}+t_k\Delta_{\parallel}\right),\\
  g_{\mu,k}
  &=q(1-\eta q)^{k-1}\Delta,
  \end{aligned}
\end{equation}
where
\begin{equation}
  t_k
  =\kappa
  \left(
    \frac{1-\eta\kappa q}{1-\eta q}
  \right)^{k-1}.
  \label{eq:gradient-direction-ratio}
\end{equation}

Suppose
\begin{equation}
  \|\Delta_{\parallel}\|_F^2
  =\frac{1}{m}\|\Delta\|_F^2.
\end{equation}
Then the angle $\psi_k$ between the gradients satisfies
\begin{equation}
  \cos\psi_k
  =\frac{\operatorname{tr}(g_{0,k}g_{\mu,k}^\top)}
  {\|g_{0,k}\|_F\|g_{\mu,k}\|_F}
  =\frac{m-1+t_k}
  {\sqrt{m\left(m-1+t_k^2\right)}}.
  \label{eq:parameter-gradient-angle}
\end{equation}

Using the ReLU calculation in Appendix~\ref{app:optimal-b},
\begin{equation}
  \mu=\frac{1}{\sqrt{6\pi}},
  \qquad
  q=\mathbb{E}[\operatorname{ReLU}(Z)^2]-\mu^2
  =\frac{1}{6}-\mu^2
  =\frac{\pi-1}{6\pi}.
\end{equation}
Table~\ref{tab:parameter-gradient-angles} compares our one-layer theory prediction
with full-model gradient angles measured during Adam training.

\begin{table}[h]
  \centering
  \caption{Angles between the unshifted and shifted parameter gradients.
  The theoretical and measured angles are close early in training.}
  \label{tab:parameter-gradient-angles}
  \begin{tabular}{rcc}
    \toprule
    Step $k$ & Theory angle & Full-model angle \\
    \midrule
    1    & $77.45^\circ$ & $76.25^\circ$ \\
    5    & $77.44^\circ$ & $75.01^\circ$ \\
    20   & $77.40^\circ$ & $69.22^\circ$ \\
    100  & $77.22^\circ$ & $46.58^\circ$ \\
    \bottomrule
  \end{tabular}
\end{table}

Both the one-layer calculation and the full-model measurements show that
the gradients of unshifted and shifted models are nearly orthogonal at the beginning of
training.

\section{Ablation studies}
\label{app:toy-model-ablations}

This section tests whether the effect of activation bias is specific to the
model configuration. We first examine a joint sweep over
model size, then study how the optimal shifting changes with model size,
and finally vary other configurations of the controlled model.

\subsection{Joint sweep over dimension, width, and depth}
\label{app:joint-scale-sweep}

We first compare shifted and unshifted ReLU models across a joint sweep of
residual-stream dimension $d$, MLP hidden width $m$, and number of blocks $L$.
Figure~\ref{fig:architecture-grid} shows the results for $L=10$ and $L=20$.
The improvement from reducing activation bias is generally larger at greater
depth and width.

\begin{figure}[!htbp]
  \centering
  \includegraphics[width=\linewidth]{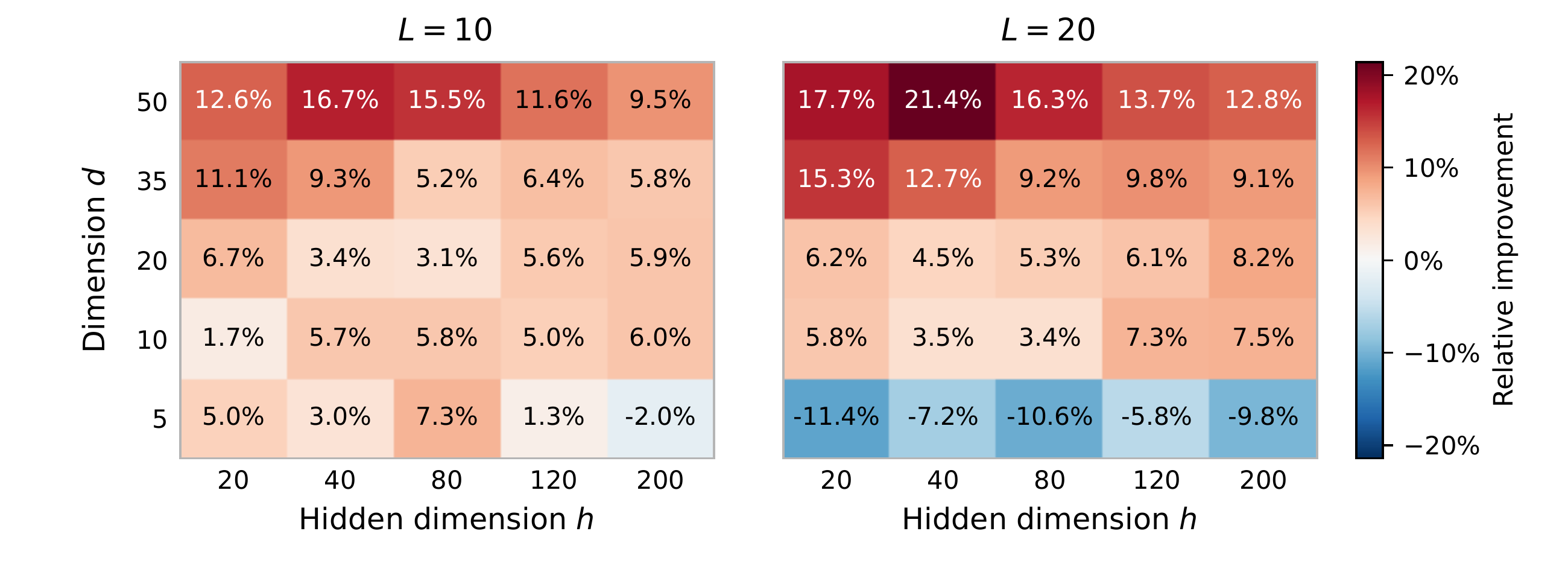}
  \caption{Shifting tends to yield greater improvements as model width and depth increase.
  For ReLU, colors show the relative improvement
  $(\mathcal L_{\mathrm{unshifted}}-\mathcal L_{\mathrm{shifted}})/
  \mathcal L_{\mathrm{unshifted}}$, with red indicating improvement and blue
  indicating deterioration.}
  \label{fig:architecture-grid}
\end{figure}
\newpage

\subsection{Seed-level optimal shifts across model sizes}
\label{app:optimal-b-across-scales}

We next compare three representative controlled ReLU models: a small model
$(d,m,L)=(5,20,2)$, a medium model $(20,80,10)$, and a large model
$(50,200,20)$. We use 100 paired seeds for each model. For each seed,
the empirically optimal shift is the value of $b$ with the lowest final loss.
Figure~\ref{fig:optimal-b-across-scales} shows the distribution of these shifts
across seeds for each model size. As model size increases, the optimal shifts
concentrate near the theory-predicted value, suggesting that reducing
post-activation bias becomes more consistently beneficial across seeds.

\begin{figure}[!htbp]
  \centering
  \includegraphics[width=\linewidth]{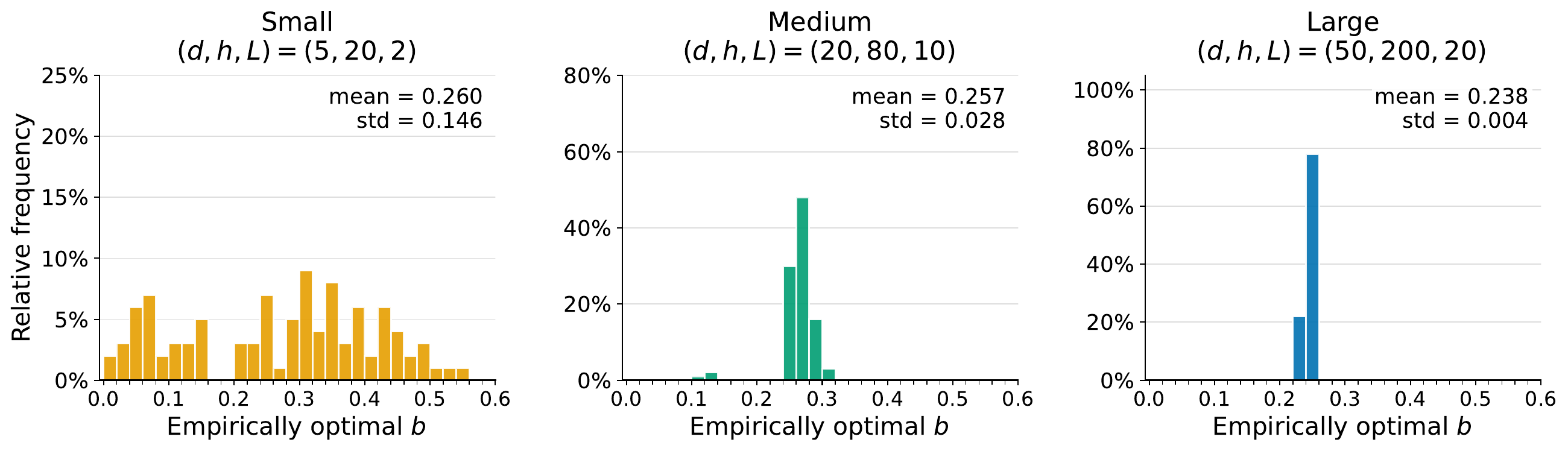}
  \caption{As model size increases, the empirically optimal shifts across seeds
  become more concentrated near the theory-predicted shift.}
  \label{fig:optimal-b-across-scales}
\end{figure}
\FloatBarrier

\subsection{Robustness to architectural changes and cross-entropy loss}
\label{app:architectural-robustness}

We finally test whether reducing activation bias remains effective under three
changes to the controlled model. First, we replace pre-norm blocks with
post-norm blocks, moving LayerNorm from the block input to the residual-sum
output. This tests whether a fixed shift remains effective when the
distributions entering the activations can vary more across blocks. Second, we
enable trainable biases in the teacher's and students' linear projections to
test whether learned affine offsets can compensate for activation bias. Third,
we replace mean-squared error with soft-target cross-entropy to test a different
training objective. All three variants use ReLU, with all other settings
unchanged.

Figure~\ref{fig:ablation-loss-residuals} shows that the shifted model achieves
lower loss than the unshifted model in all three variants. Thus, its advantage
is not limited to pre-norm blocks, bias-free projections, or mean-squared error.

\begin{figure}[!htbp]
  \centering
  \includegraphics[width=\linewidth]{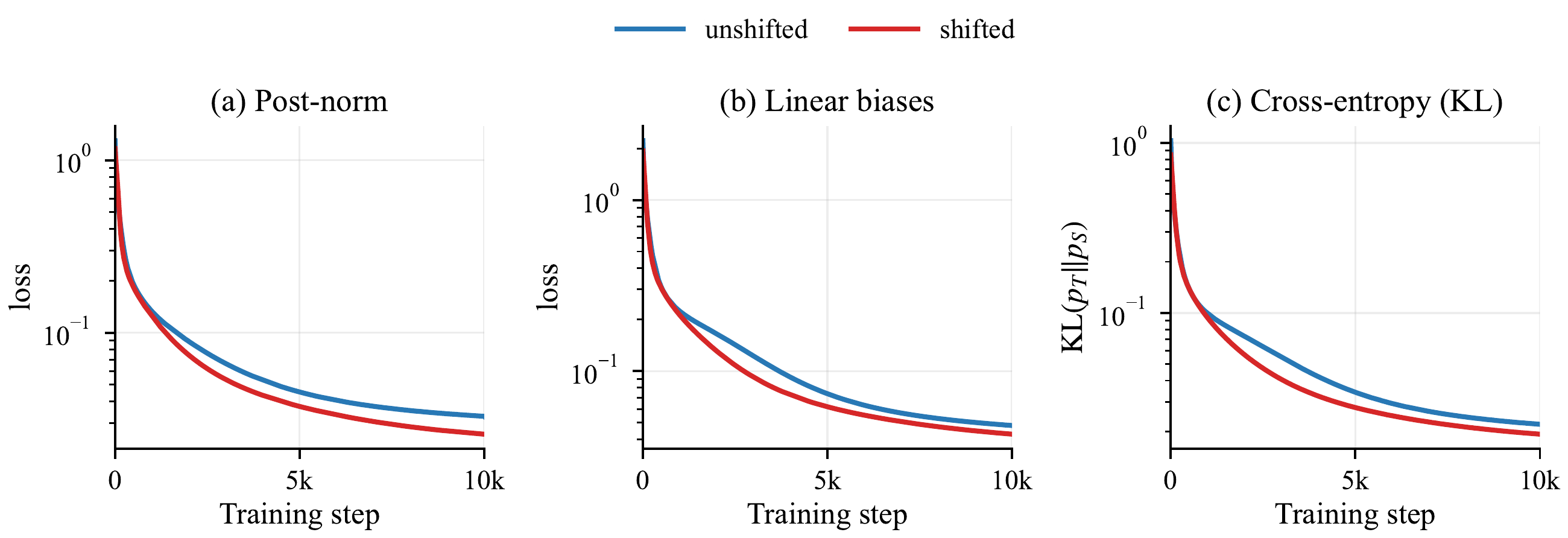}
  \caption{The shifted model retains its loss advantage in more general settings.}
  \label{fig:ablation-loss-residuals}
\end{figure}
\FloatBarrier
\subsection{Residual-stream trajectories across all blocks}
\label{app:residual-stream-all-blocks}

Figure~\ref{fig:toy-residual-all-blocks} shows the complete hidden-state mean
trajectories across all 20 blocks for the ReLU experiment in
Figure~\ref{fig:toy-loss-overshoot}. The unshifted model exhibits a pronounced
transient overshoot across most of the blocks, whereas shifting the
activation substantially reduces this overshoot. 
\begin{figure}[!htbp]
  \centering
  \includegraphics[width=\linewidth]{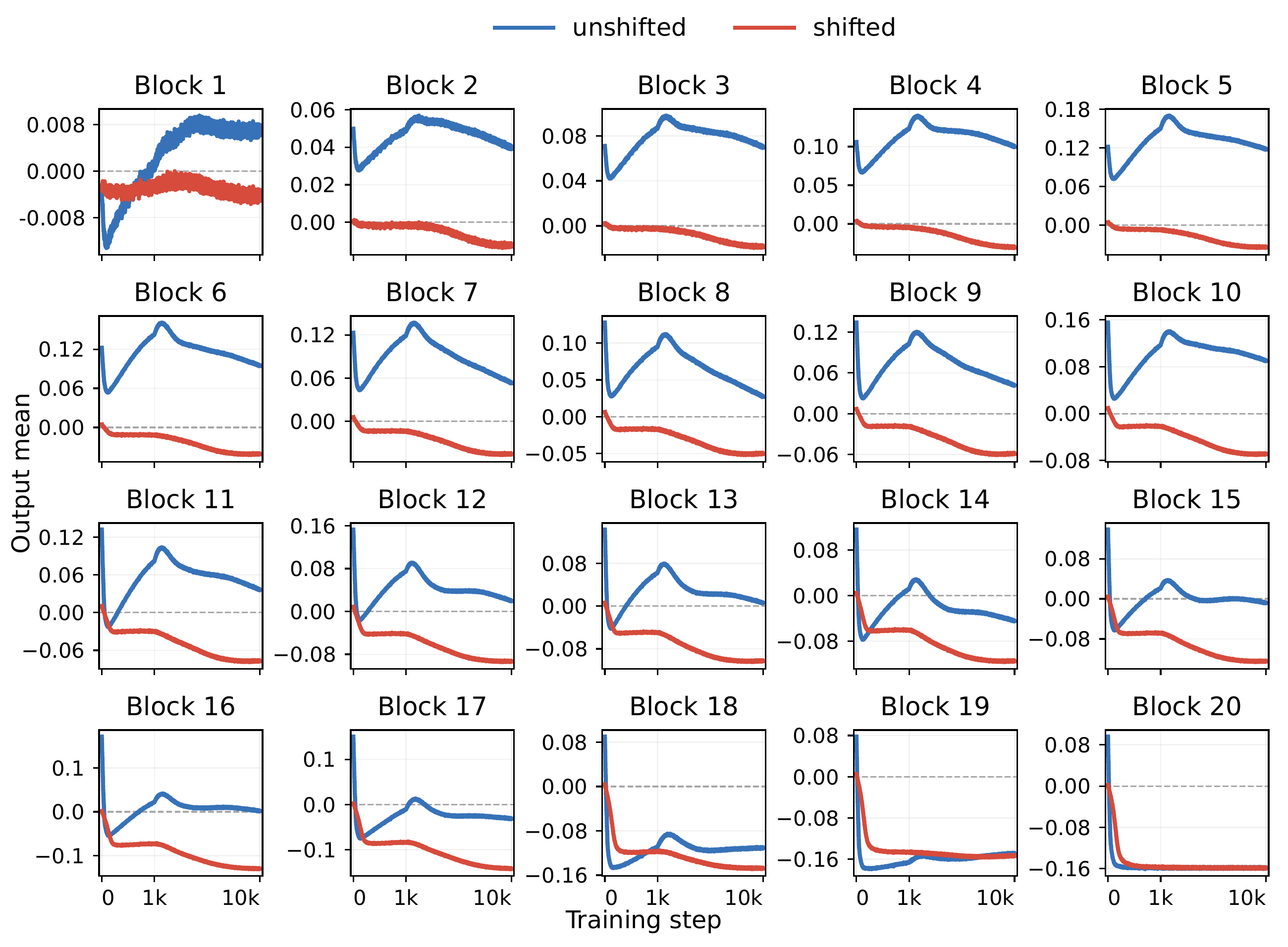}
  \caption{Reducing activation bias suppresses
  hidden-state overshoot and yields smoother hidden-state mean trajectories.}
  \label{fig:toy-residual-all-blocks}
\end{figure}
\clearpage

\section{Additional VGG-Style Autoencoder Experiments}
\label{app:vgg-reconstruction}

This appendix provides the architecture and additional experimental results for the simplified VGG-style
autoencoder introduced in Section~\ref{sec:feature-learning-setup}.
Figure~\ref{fig:vgg-autoencoder-architecture} shows the architecture.

\begin{figure}[htbp]
  \centering
  \includegraphics[width=\linewidth]{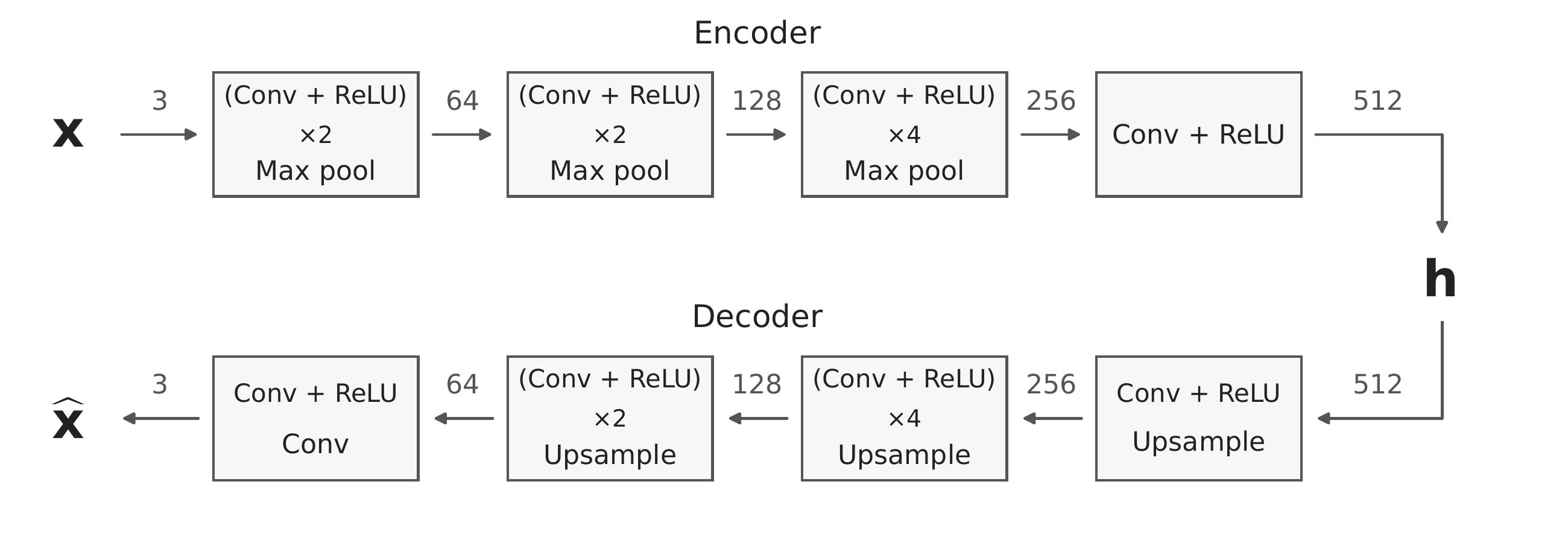}
  \caption{Architecture of the autoencoder. Numbers on arrows denote
  channel counts. The shifted model replaces each ReLU
  with $\operatorname{ReLU}(z)-b$.}
  \label{fig:vgg-autoencoder-architecture}
\end{figure}

We also track the reconstruction loss throughout training.
Figure~\ref{fig:vgg-loss-residual} shows that the shifted model initially has
a higher loss than the unshifted model but achieves a lower loss later.
\begin{figure}[htbp]
  \centering
  \includegraphics[width=0.55\linewidth]{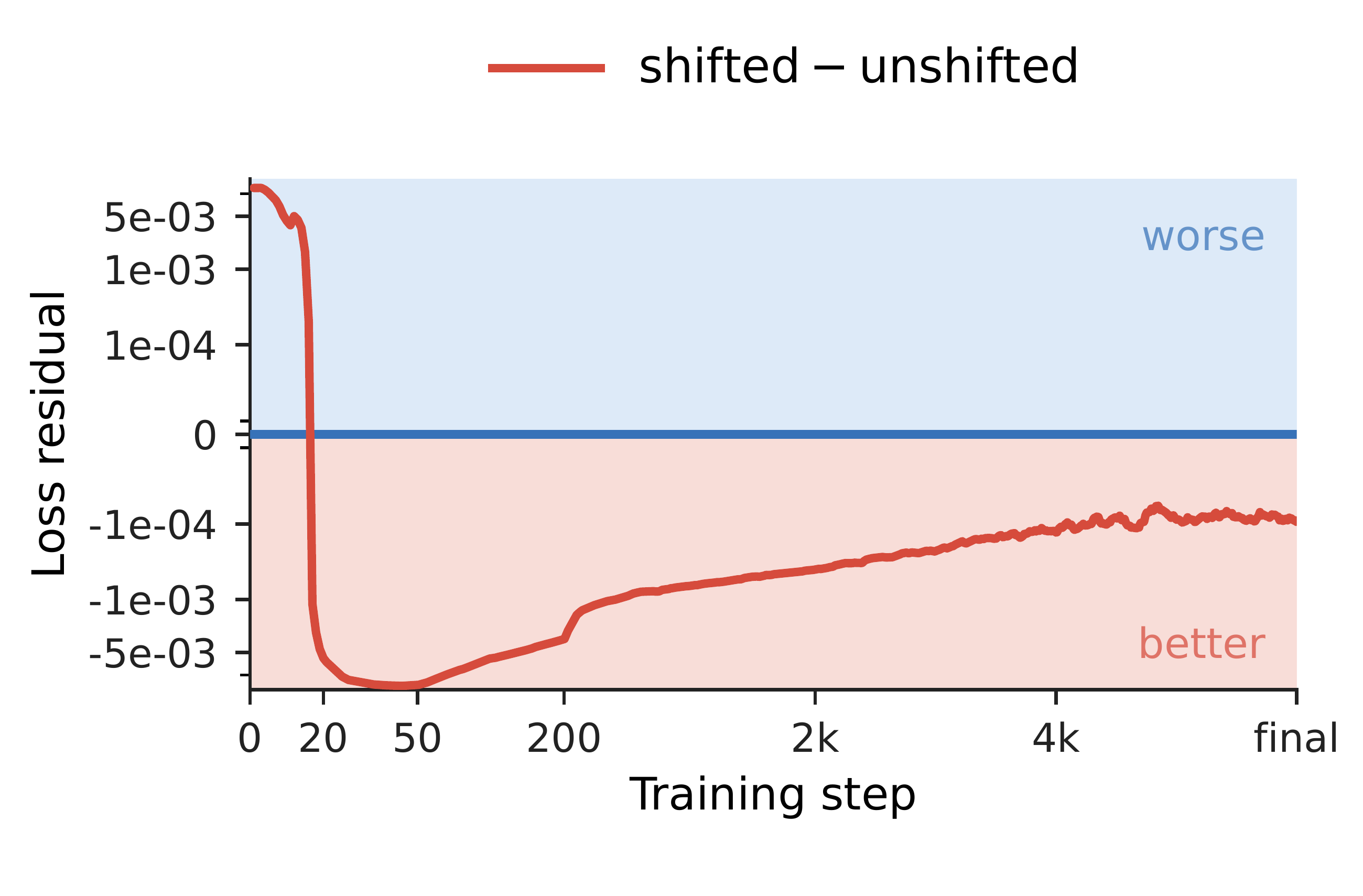}
  \caption{The shifted model initially has a higher reconstruction loss but
  achieves a lower loss later in training. The plotted residual is
  $L_{\mathrm{shifted}}-L_{\mathrm{unshifted}}$.}
  \label{fig:vgg-loss-residual}
\end{figure}

Figure~\ref{fig:vgg-reconstruction-examples} provides additional examples of
the component-wise reconstructions analyzed in
Section~\ref{sec:bias-component-analysis}. Across these images, the unshifted
model's activation bias component retains recognizable image structure, whereas the
corresponding reconstructions from the shifted model contain little
discernible content.

\clearpage
\begin{figure}[p]
  \centering
  \includegraphics[width=\linewidth]{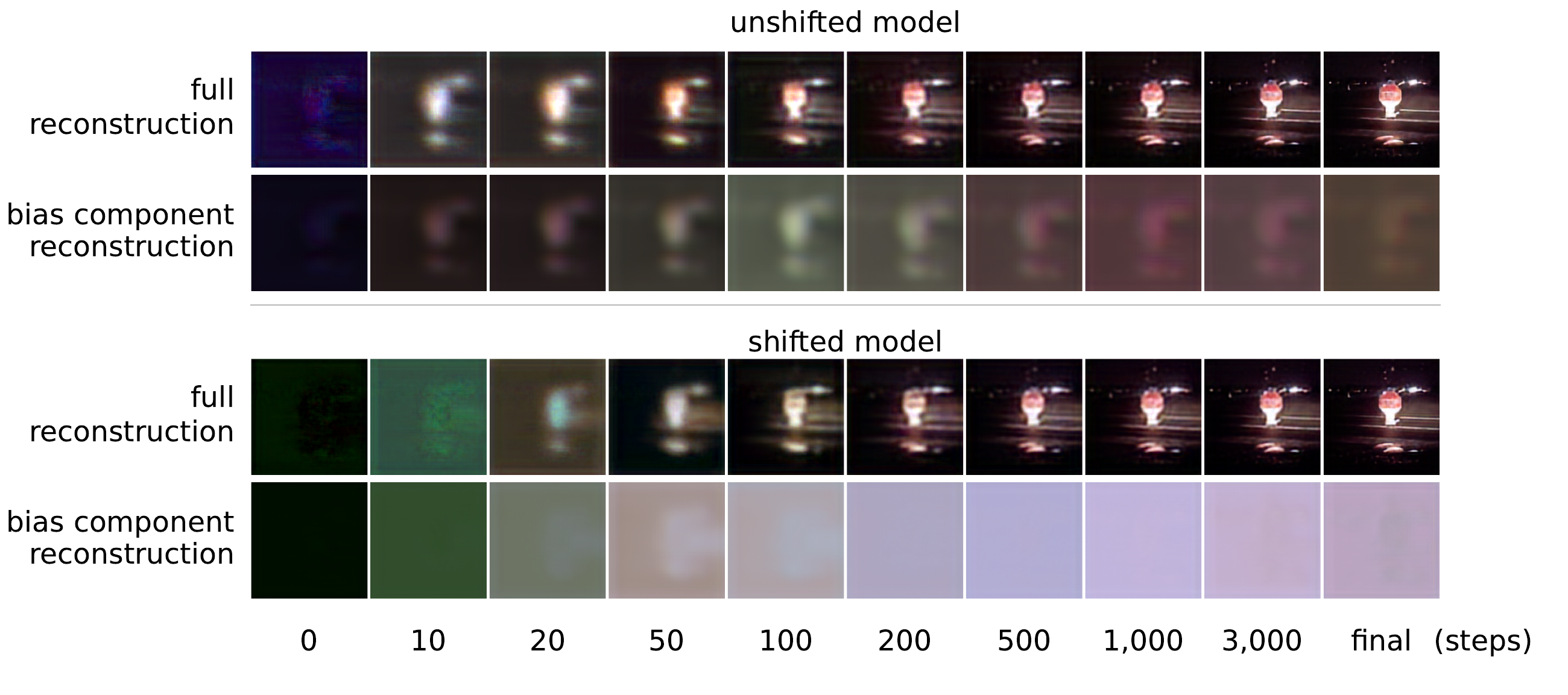}
  \par\vspace{0.2em}

  \includegraphics[width=\linewidth]{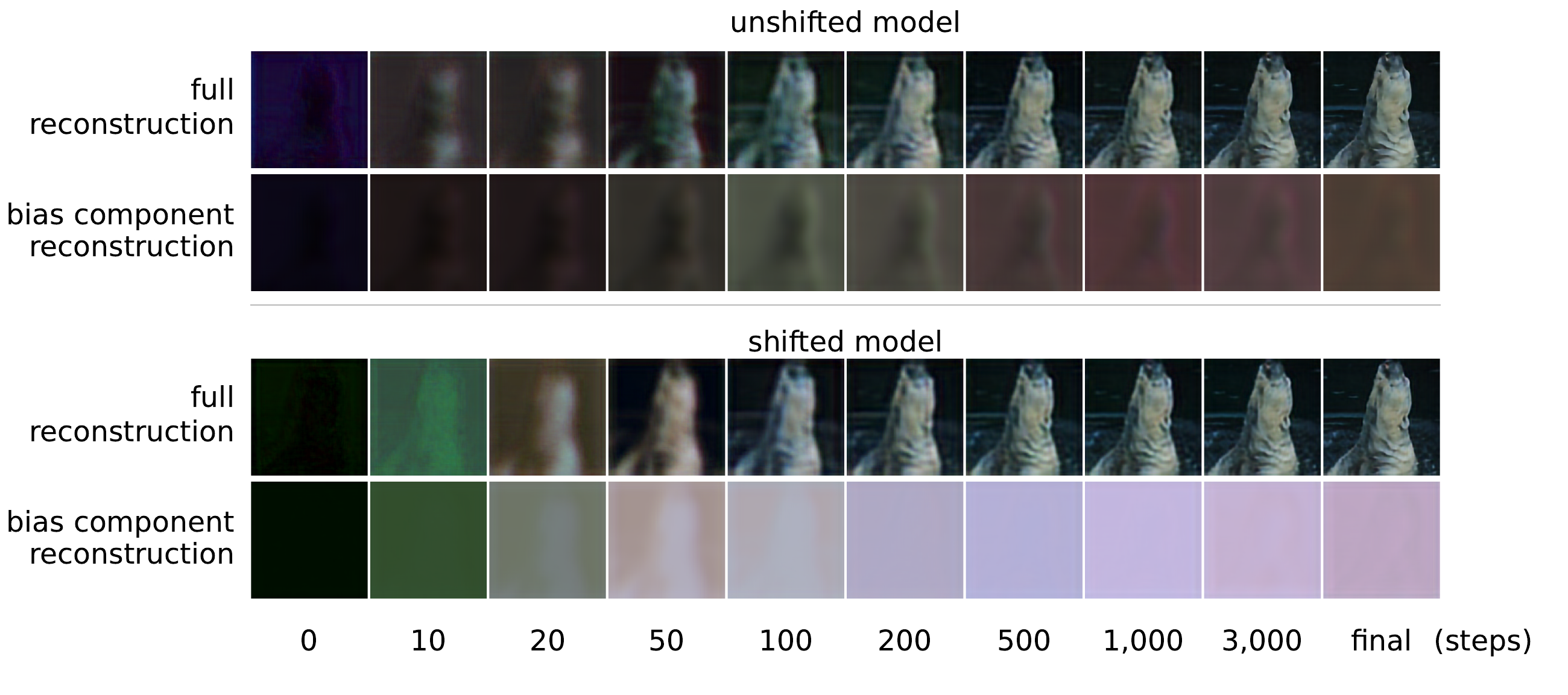}
  \par\vspace{0.2em}

  \includegraphics[width=\linewidth]{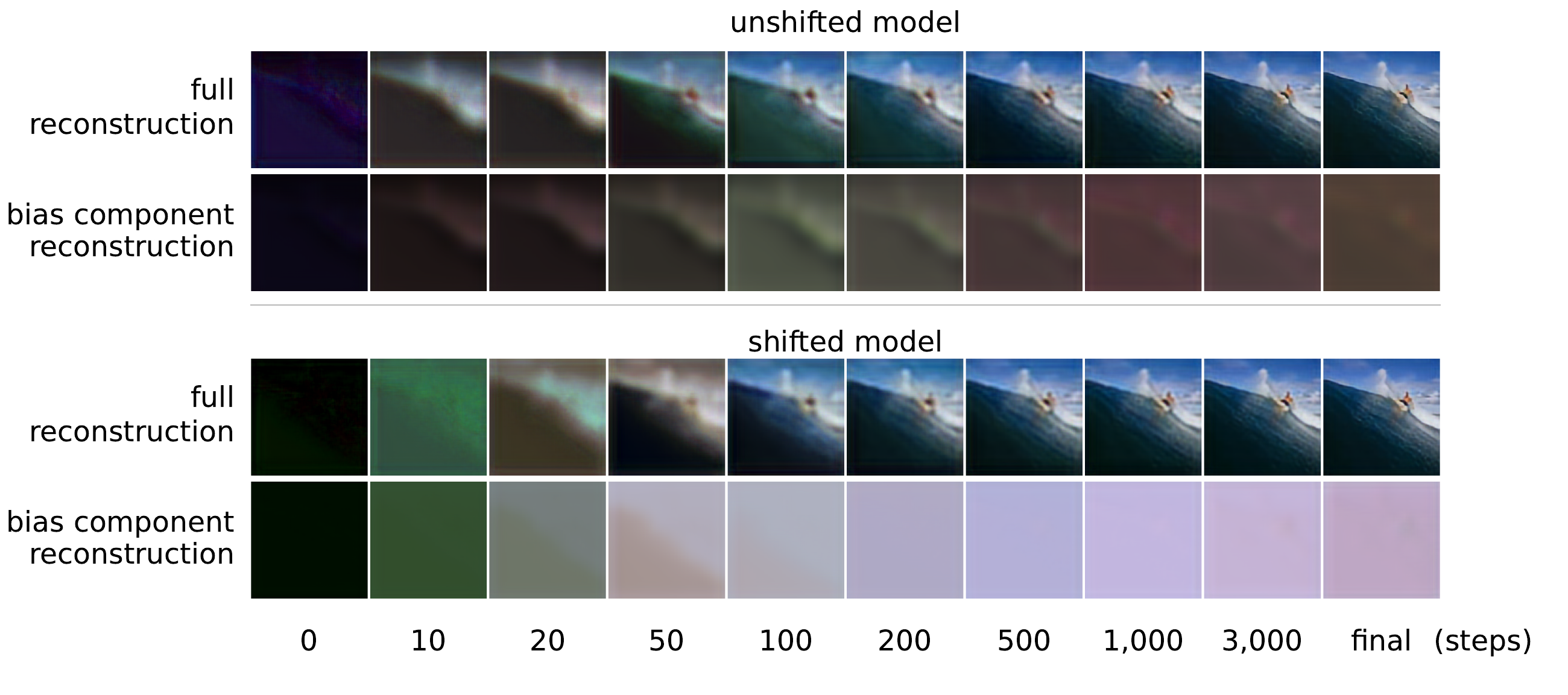}
  \caption{Across different images, reconstructions from the unshifted model's
  activation bias component first approach and then diverge from the full reconstructions,
  whereas those from the shifted model contain little recognizable structure.}
  \label{fig:vgg-reconstruction-examples}
\end{figure}
\clearpage

\section{Additional Scale-Up Experiments}
\label{app:additional-scaleup}

This appendix provides additional scale-up results on performance, hidden-state
dynamics, and activation distributions.

\subsection{Performance on Additional Scale-Up Models}
\label{app:additional-scaleup-losses}

We test whether the gains reported in Section~\ref{sec:scale-up} extend across
additional model sizes, tasks, and metrics. We evaluate ViT-L on ImageNet image
classification, OLMo-2-1B~\citep{olmo2025olmo2} on language pretraining,
and HyenaDNA~\citep{nguyen2023hyenadna} on DNA base-sequence prediction.
We also evaluate DiT-B/2 on image generation using
FID. Figure~\ref{fig:additional-scaleup-losses} reports the results. With
adaptive shifting, ViT-L achieves lower test loss and higher accuracy
despite a higher training loss; DiT achieves lower FID at every evaluated
checkpoint, while OLMo-2-1B and HyenaDNA achieve better performance. Thus, the
benefit extends beyond the models and metrics highlighted in the main text.

\begin{figure}[!htbp]
  \centering
  \includegraphics[width=\linewidth]{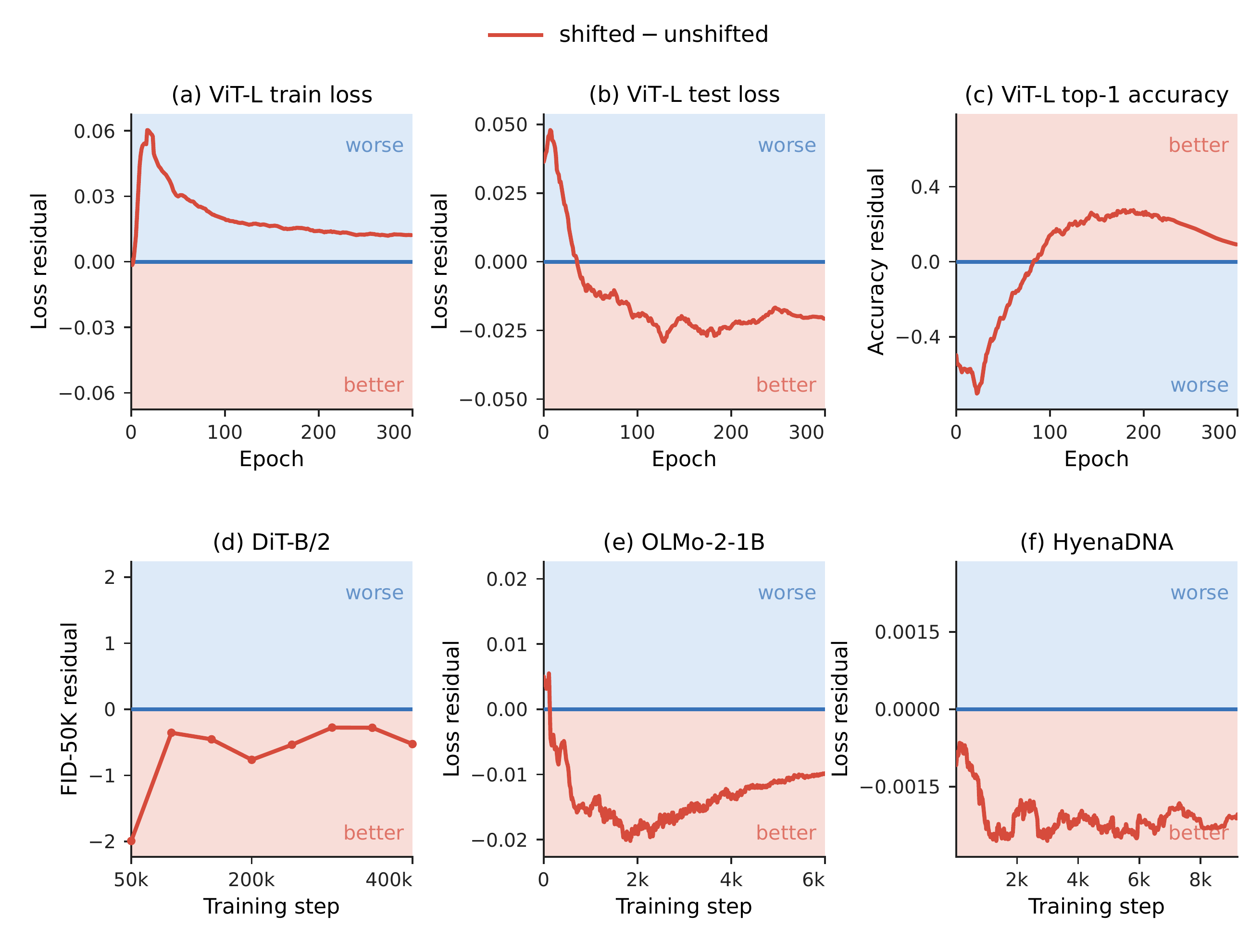}
  \caption{The performance improvements from shifting extend to additional
  scale-up models and tasks. Residuals are computed as
  $\Delta M=M_{\mathrm{shifted}}-M_{\mathrm{unshifted}}$, where $M$ denotes
  loss, FID, or accuracy.}
  \label{fig:additional-scaleup-losses}
\end{figure}

\subsection{Further Examples of Hidden-State Overshoot}
\label{app:scaleup-overshoot}

We examine whether the hidden-state overshoot observed in
Section~\ref{sec:scale-up-vit} also appears in other models. We evaluate publicly
released training checkpoints from Pythia-410M~\citep{biderman2023pythia},
OLMo-2-1B~\citep{olmo2025olmo2}, and
Amber-7B~\citep{liu2023llm360} on a fixed set of inputs and compute
hidden-state means as in Section~\ref{sec:hidden-state-overshoot}.
Figure~\ref{fig:official-checkpoint-residual-probes} presents representative
trajectories from these models. Across them,
representative blocks exhibit transient overshoots that subsequently recede or
reverse, although their timing and direction vary. 

\begin{figure}[!htbp]
  \centering
  \includegraphics[width=\linewidth]{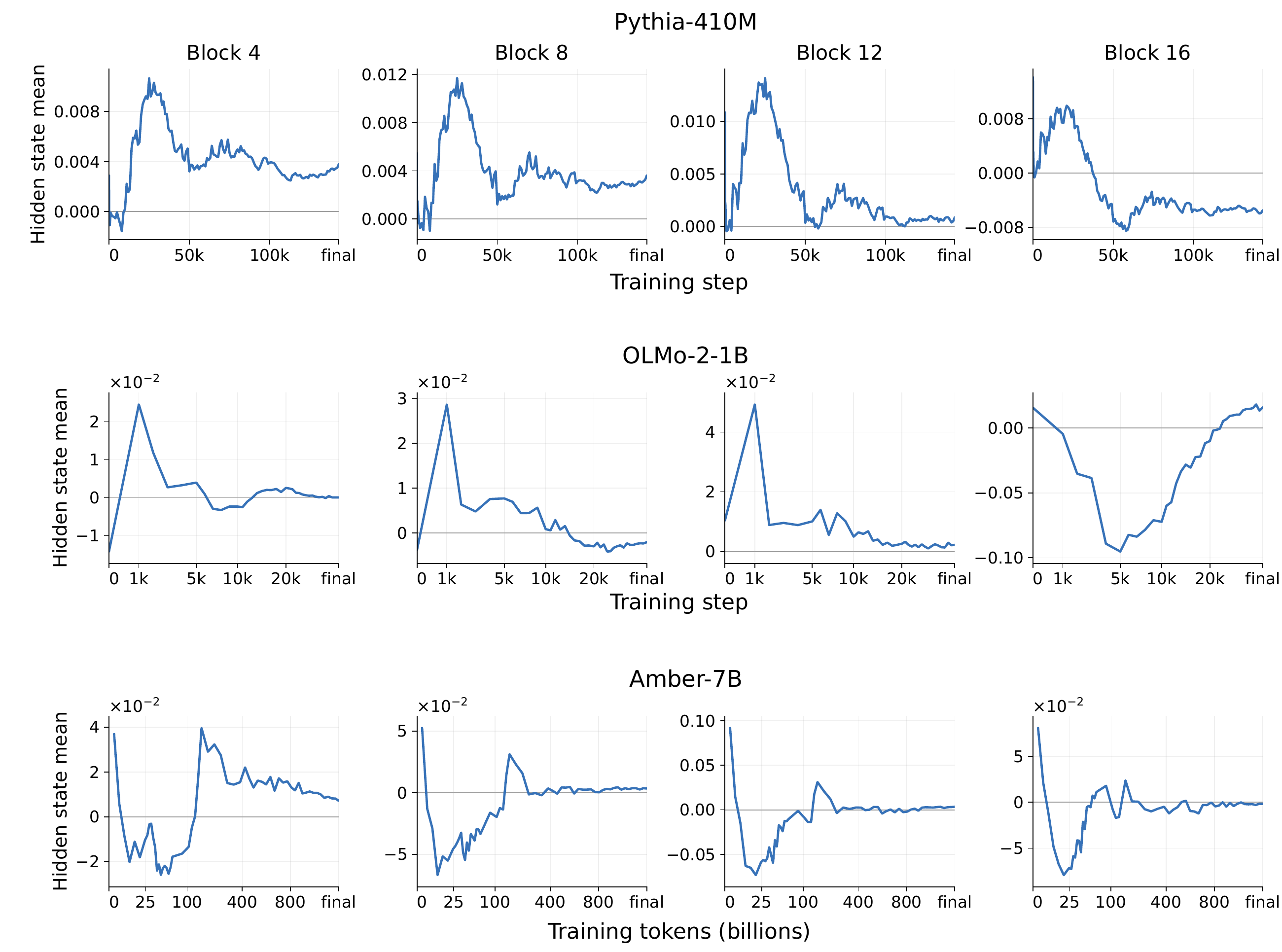}
  \caption{Hidden-state mean trajectories exhibit overshoot across different
  scale-up models.}
  \label{fig:official-checkpoint-residual-probes}
\end{figure}

\subsection{Activation-Magnitude Distributions in OLMo-2-1B}
\label{app:scaleup-activation-distributions}

Our analysis so far has focused on activation bias, but the mean alone cannot
show how shifting affects the broader distribution of activation magnitudes.
We therefore compare the absolute-value distributions of post-activations and
residual-stream states. Figure~\ref{fig:olmo-absolute-value-distributions} shows
that, with shifting, the post-activation distributions change relatively little
over training. In the unshifted model, they develop an increasingly pronounced
concentration near zero in deeper blocks. The residual-stream distributions,
however, remain broadly similar across the two models and more closely resemble
the shifted model's post-activation distributions. This contrast suggests that
the unshifted model must map a more sharply concentrated post-activation
distribution to a comparatively smooth residual-stream distribution. Such a
larger distributional adjustment may make learning less direct, offering a
possible explanation for the improved dynamics observed when activation bias
is reduced.

\begin{figure}[!htbp]
  \centering
  \includegraphics[width=\linewidth]{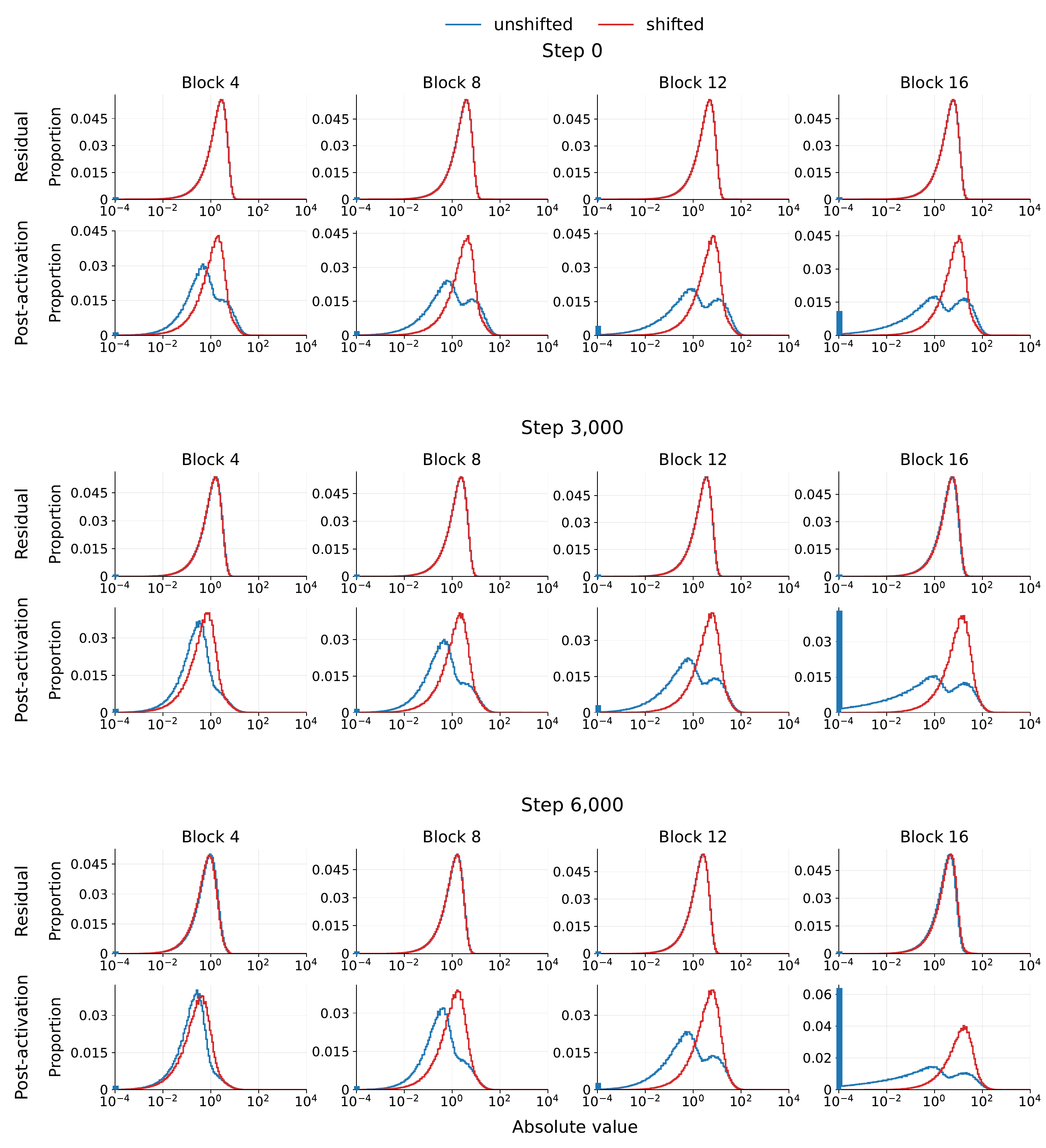}
  \caption{In OLMo-2-1B, the shifted model maintains more stable post-activation
  magnitude distributions over training, with less accumulation of near-zero
  values. Moreover, its post-activation and residual-stream magnitude
  distributions closely resemble each other, whereas those of the unshifted
  model differ substantially.}
  \label{fig:olmo-absolute-value-distributions}
\end{figure}

\end{document}